\documentclass{article}

\RequirePackage{amsmath}

\usepackage[preprint,nonatbib]{nips_2018}

\usepackage{graphicx}
\usepackage{amssymb}
\usepackage{subfigure}
\usepackage{algorithm}
\usepackage{algorithmic}
\usepackage{tikz}
\usepackage{fancyhdr}
\usepackage{color}
\usepackage{bm}
\usepackage{comment}
\usepackage{amsthm}
\usepackage[square,sort,comma,numbers]{natbib}

\usepackage{placeins} 
\usepackage{multirow}
\usepackage{mathtools}

\newtheorem{assumption}{Assumption}
\newtheorem{remark}{Remark}
\newtheorem{theorem}{Theorem}
\newtheorem{lemma}{Lemma}
\newtheorem{corollary}{Corollary}
\newcommand{\Av}{\mathbf{A}}

\newcommand{\Sv}{\mathbf{S}}

\newcommand{\yv}{\mathbf{y}}
\newcommand{\zv}{\mathbf{z}}

\newcommand{\xv}{\mathbf{x}}

\newcommand{\uv}{\mathbf{u}}

\newcommand{\dv}{\mathbf{d}}

\newcommand{\Wve}{\mathbf{W_e}}
\newcommand{\Wv}{\mathbf{W}}
\newcommand{\Lc}{\mathcal{L}}
\newcommand{\hLc}{\hat{\Lc}_A}

\newcommand{\Wlista}{\tilde{\mathbf{W}}_{\mathbf{e}}}
\newcommand{\Slista}{\tilde{\Sv}}
\newcommand{\Lu}[1]{\mathcal{L}_{#1}}
\newcommand{\phiv}{\boldsymbol{\phi}}
\newcommand{\thetav}{\boldsymbol{\theta}}

\newcommand{\RR}{\mathbb{R}}
\newcommand{\norm}[1]{\left\lVert#1\right\rVert}
\newcommand{\argmin}[1]{\text{arg}\min_{#1}}
\newcommand{\iter}[1]{^{(#1)}}
\newcommand{\Tt}{^{\text{T}}}
\newcommand{\vecTwo}[2]{\begin{bmatrix}#1\\#2\end{bmatrix}}
\newcommand{\centerIm}[2]{\raisebox{-.5\height}{\includegraphics[width = #1]{#2}}}
\newcommand{\soft}{\text{soft}}
\newcommand{\E}[1]{\mathbb{E}\left[#1\right]}

\begin{document}
\title{LSALSA: Accelerated Source Separation via Learned Sparse Coding}




\author{Benjamin Cowen \and Apoorva Nandini Saridena \and Anna Choromanska
\and
\hspace{-0.5cm}ben.cowen@nyu.edu\and 
\hspace{0.25cm}ans609@nyu.edu \and 
\hspace{1cm}ac5455@nyu.edu }



\date{Received: date / Accepted: date}

\maketitle

\begin{abstract}
We propose an efficient algorithm for the generalized sparse coding (SC) inference problem. The proposed framework applies to both the single dictionary setting, where each data point is represented as a sparse combination of the columns of one dictionary matrix, as well as the multiple dictionary setting as given in morphological component analysis (MCA), where the goal is to separate a signal into additive parts such that each part has distinct sparse representation within an appropriately chosen corresponding dictionary. Both the SC task and its generalization via MCA have been cast as $\ell_1$-regularized optimization problems of minimizing quadratic reconstruction error. In an effort to accelerate traditional acquisition of sparse codes, we propose a deep learning architecture that constitutes a trainable time-unfolded version of the Split Augmented Lagrangian Shrinkage Algorithm (SALSA), a special case of the Alternating Direction Method of Multipliers (ADMM). We empirically validate both variants of the algorithm, that we refer to as LSALSA (learned-SALSA), on image vision tasks and demonstrate that at inference our networks achieve vast improvements in terms of the running time and the quality of estimated sparse codes on both classic SC and MCA problems over more common baselines. We also demonstrate the visual advantage of our technique on the task of source separation.
Finally, we present a theoretical framework for analyzing LSALSA network: we show that the proposed approach exactly implements a truncated ADMM applied to a new, learned cost function with curvature modified by one of the learned parameterized matrices. We extend a very recent Stochastic Alternating Optimization analysis framework to show that a gradient descent step along this learned loss landscape is equivalent to a modified gradient descent step along the original loss landscape. In this framework, the acceleration achieved by LSALSA could potentially be explained by the network's ability to learn a correction to the gradient direction of steeper descent.
\end{abstract}

\newpage
\section{Introduction}
\label{intro}
In the SC framework, we seek to efficiently represent data by using only a sparse combination of available basis vectors. We therefore assume that an $M$-dimensional data vector $\yv\in\RR^M$ can be approximated as
\vspace{-0.25cm}
\begin{align}\label{model: lasso}
    \yv \approx \Av\xv^*,
\end{align}
where $\xv^*\in\RR^N$ is sparse and  $\Av\in\RR^{M\times N}$ is a dictionary, sometimes referred to as the synthesis matrix, whose columns are the basis vectors. This paper focuses on the generalized SC problem of decomposing a signal into morphologically distinct components. A typical assumption for this problem is that the data is a linear combination of $D$ source signals:
\vspace{-0.25cm}
\begin{align}\label{model: sigSep}
    \yv = \sum_{i=1}^D\yv_i.
\end{align}

The MCA framework~\cite{mca:orig_starck05} for addressing additive mixtures requires that each component $\yv_i$ admits a sparse representation within the corresponding dictionary $\Av_i$, leading to a generalized signal approximation model:
\vspace{-0.25cm}
\begin{align}\label{model: mca} 
\yv \approx \sum_{i=1}^D\Av_i\xv_i^*.
\end{align}
We then seek to recover $x_i^{*}$s given $y$ and dictionaries $A_i$s. We may trivially satisfy (\ref{model: mca}) by setting, for example, $\xv^*_i=0$ for all $i\neq j$, and performing traditional SC using only dictionary $\Av_j$. Thus, MCA further assumes that the dictionaries $\Av_i$s are \textit{distinct} in the sense that each source-specific dictionary allows obtaining sparse representation of the corresponding source signal, while being highly inefficient in representing the other content in the mixture. This assumption is difficult to enforce on harder problems, i.e. when the components $\yv_i$ have similar characteristics and do not admit intuitive \textit{a priori} sparsifying bases. In practice, the $\Av_i$s often have significant overlap in sparse representation, making the problem of jointly recovering the $\xv_i$s highly ill-conditioned.

There exist iterative optimization algorithms for performing SC and MCA. The bottleneck of these techniques is that at inference a sparse code has to be computed for each data point or data patch (as in case of high-resolution images). In the single dictionary setting, ISTA\cite{cvxOpt:daub04} and FISTA~\cite{Beck:2009:FIS:1658360.1658364} are classical algorithmic choices for this purpose. For the MCA problem, the standard choice is SALSA~\cite{Afonso:2011:ALA:2319096.2321825}, an instance of ADMM\cite{cvxOpt:boyd11}. The iterative optimization process is prohibitively slow for high-throughput real-time applications, especially in the case of the ill-conditioned MCA setting. Thus our goal is to provide algorithms performing efficient inference, i.e. algorithms that find good approximations of the optimal codes in significantly shorter time than FISTA or SALSA. 

The first key contribution of this paper is an efficient and accurate deep learning architecture that is general enough to well-approximate optimal codes for both classic SC in a single-dictionary framework and MCA-based signal separation. By accelerating SALSA via learning, we provide a means for fast approximate source separation. We call our deep learning approximator Learned SALSA (LSALSA). The proposed encoder is formulated as a time-unfolded version of the SALSA algorithm with a fixed number of iterations, where the depth of the deep learning model corresponds to the number of SALSA iterations. We train the deep model in the supervised fashion to predict optimal sparse codes for a given input and show that shallow architectures of fixed-depth, that correspond to only few iterations of the original SALSA, achieve superior performance to the classic algorithm.

The SALSA algorithm uses second-order information about the cost function, which gives it an advantage over popular comparators such as ISTA on ill-conditioned problems~\cite{salsa:istaCompare}. Our second key contribution is an empirical demonstration that this advantage carries over to the deep-learning accelerated versions LSALSA and LISTA~\cite{DBLP:conf/icml/GregorL10}, while preserving SALSA's applicability to a broader class of learning problems such as MCA-based source separation (LISTA is used only in the single dictionary setting).
To the best of our knowledge, our approach is the first one to utilize an instance of ADMM unrolled into a deep learning architecture to address a source separation problem

Our third key contribution is a theoretical framework that provides insight into how LSALSA is able to surpass SALSA, namely describing how the learning procedure can enhance the second-order information that is characteristically exploited by SALSA.
In particular, we show that the forward-propagation of a signal through the LSALSA network is equivalent to the application of truncated-ADMM to a new, learned cost function, and present a theoretical framework for characterizing this function in relation to the original Augmented Lagrangian.
To the best of our knowledge, our work is the first to attempt to analyze a learning-accelerated ADMM algorithm.

To summarize, our contributions are threefold:
\begin{enumerate}
	\item We achieve significant acceleration in both SC and MCA: classic SALSA takes up to $100\times$ longer to achieve LSALSA's performance. This opens up the MCA framework to potentially be used in high-throughput, real-time applications.
	\item We carefully compare an ADMM-based algorithm (SALSA) with our proposed learnable counterpart (LSALSA) and with popular baselines (ISTA and FISTA). For a large variety of computational constraints (i.e. fixed number of iterations), we perform comprehensive hyperparameter testing for each encoding method to ensure a fair comparison.
	\item We present a theoretical framework for analyzing the LSALSA network, giving insight as to how it uses information learned from data to accelerate SALSA.
\end{enumerate}

This paper is organized as follows: Section~\ref{sec:rw} provides literature review, Section~\ref{section: background} formulates the SC problem in detail, Section~\ref{section: admm/salsa} shows how to
derive predictive single dictionary SC and multiple dictionary MCA from their iterative counterparts and explains our approach (LSALSA). Section~\ref{sec:analysis} elaborates our theoretical framework for analyzing LSALSA and provides insight into its empirically demonstrated advantages. Section~\ref{section: results} shows experimental results for both the single dictionary setting and MCA. Finally, Section~\ref{sec:conclusions} concludes the paper. We provide an open-source implementation of the sparse coding and source separation experiments presented herein.

\vspace{-0.4cm}
\section{Related Work} 
\label{sec:rw}
A sparse code inference aims at computing sparse codes for given data and is most widely addressed via iterative schemes such as aforementioned ISTA and FISTA. Predicting approximations of optimal codes can be done using deep feed-forward learning architectures based on truncated convex solvers. This family of approaches lies at the core of this paper. A notable  approach in this family known as LISTA~\cite{DBLP:conf/icml/GregorL10} stems from earlier predictive sparse decomposition methods~\cite{koray-psd-08,conf/iccv/JarrettKRL09}, which however were obtaining approximations to the sparse codes of insufficient quality. LISTA improves over these techniques and enhances ISTA by unfolding a fixed number of iterations to define a fixed-depth deep neural network that is trained with examples of input vectors paired with their corresponding optimal sparse codes obtained by conventional methods like ISTA or FISTA. LISTA was shown to provide high-quality approximations of optimal sparse codes with a fixed computational cost.
Unrolling methodology has since been applied to algorithms solving SC with $\ell_0$-regularization~\cite{sparseCoding:wang16} and message passing schemes~\cite{sparseCoding:borgerding16}. In other prior works, ISTA was recast as a recurrent neural network unit giving rise to a variant of LSTM~\cite{Gers:2003:LPT:944919.944925,sparseCoding:zhou18}. Recently, theoretical analysis has been provided for LISTA~\cite{listatheory:chen2018,listatheory:moreau2016}, in which the authors provide convergence analyses by imposing constraints on the LISTA algorithm. This analysis does not apply to the MCA problem as it cannot handle multiple dictionaries. In other words, they would approach the MCA problem by casting it as a SC problem with access to a single dictionary that is a concatenation of source-specific dictionaries, e.g. $[\Av_1,\Av_2,\dots,\Av_D]$. Furthermore these analyses do not address the saddle-point setting as required for ADMM-type methods such as SALSA.

MCA has been used successfully in a number of applications that include decomposing images into textures and cartoons for denoising and inpainting~\cite{mca:elad05,mca:peyre07,mca:peyre10,mca:shoham08,mca:starck05_2,mca:starck05_1}, detecting text in natural scene images~\cite{expSet:mcaText2017}, as well as other source separation problems such as separating non-stationary clutter from weather radar signals~\cite{mca:uysal16}, transients from sustained rhythmic components in EEG signals~\cite{mca:parekh14}, and stationary from dynamic components of MRI videos~\cite{mca:otazo15}. The MCA problem is frequently solved via SALSA algorithm, which constitutes a special case of the ADMM method. 

There exist a few approaches in the literature utilizing highly specialized trainable ADMM algorithms. One such framework~\cite{unrolled:yang16} was demonstrated to improve the reconstruction accuracy and inference speed over a variety of state-of-the-art solvers for the problem of compressive sensing Magnetic Resonance Imaging. A variety of papers followed up on this work for various image reconstruction tasks, such as the Learned Primal-dual Algorithm~\cite{unrolled:adler17}. However, these approaches do not give a detailed iteration-by-iteration comparison of the baseline method versus the learned method, making it difficult to understand the accuracy/speed tradeoff. Another related framework~\cite{unrolled:sprechmann13} was applied to efficiently learn task-specific (reconstruction or classification) sparse models via sparsity-promoting convolutional operators. None of the above methods were applied to the MCA or other source separation problems and moreover it is non-trivial to obtain such extensions of these works. An unrolled nonnegative matrix factorization (NMF) algorithm~\cite{unrolled:LeRoux15} was implemented as a deep network for the task of speech separation. In another work~\cite{unrolled:Wisdom2017DeepRN}, the NMF-based speech separation task was solved with an ISTA-like unfolded network. 

\vspace{-0.5cm}
\section{Problem Formulation}
\label{section: background}
This paper focuses on the inference problem in SC: given data vector $\yv$ and dictionary matrix $\Av$, we consider algorithms for finding the unique coefficient vector $\xv^*$ that minimizes the $\ell_1$-regularized linear least squares cost function:
\vspace{-0.15cm}
\begin{equation}
\label{lasso}
    \xv^* = \argmin{\xv}\left\{E_{\Av}(\xv;\yv) = \tfrac12\norm{\yv-\Av\xv}_2^2 + \alpha\norm{\xv}_1\right\},
\end{equation}
where the scalar constant $\alpha\geq0$ balances sparsity with data fidelity. Since this problem is convex, $\xv^*$ is unique and we refer to it as the optimal code for $\yv$ with respect to $\Av$. The dictionary matrix $\Av$ is usually learned by minimizing a loss function given below~\cite{dictLearning:Olshausen96}
\vspace{-0.25cm}
\begin{align}\label{eqn: dictlearning1D}
    \mathcal{L}_{\text{Dict}}(\Av) = \frac{1}{P}\sum_{p=1}^P E_{\Av}(\xv^{*,p}; \yv^p)
\end{align}
with respect to $\Av$ using stochastic gradient descent (SGD), where $P$ is the size of the training data set, $\yv^p$ is the $p^{th}$ training sample, and $\xv^{*,p}$ is the corresponding optimal sparse code. The optimal sparse codes in each iteration are obtained in this paper with FISTA. When training dictionaries, we require the columns of $\Av$ to have unit norm, as is common practice for regularizing the dictionary learning process~\cite{dictLearning:Olshausen96}, however this is not necessary for code inference.

In the MCA framework, a generalization of the cost function from Equation~\ref{lasso} is minimized to estimate $\xv_1^*,\xv_2^*,\dots,\xv_D^*$ from the model given in Equation~\ref{model: mca}. Thus one minimizes
\vspace{-0.05in}
\begin{equation}\label{eqn: mca}
E_{\Av}(\xv;\yv) = \tfrac12\norm{\yv-\Av\xv}_2^2 + \sum_{i=1}^D\alpha_i\norm{\xv_i}_1,
\end{equation}
using $\Av \coloneqq [\Av_1,\Av_2,\dots,\Av_D]\in\RR^{M\times N}$ and
\begin{equation}
    \xv \coloneqq
    \left[
        \begin{array}{c}
            \xv_1\\ \xv_2 \\ \vdots\\ \xv_D
        \end{array}
    \right]\in\RR^N,
\end{equation}
where $\xv_i \in \RR^{N_i}$ for $i = \{1,2,\dots,D\}$, $N = \sum_{i=1}^D N_i$, and $\alpha_i$s are the coefficients controlling the sparsity penalties. We denote the concatenated optimal codes with $\xv^* = \argmin{\xv}E_{\Av}(\xv,\yv)$.
To recover the single dictionary case, simply set $\alpha_i=\alpha_j,\ \forall i,j=1,...,D$ and set $\Av_i$ to be partitions of $\Av$.

In the classic MCA works, the dictionaries $\Av_i$s are selected to be well-known filter banks with explicitly designed sparsification properties. 
Such hand-designed transforms have good generalization abilities and help to prevent overfitting. Also, MCA algorithms often require solving large systems of equations involving $\Av\Tt\Av$ or $\Av\Av\Tt$. An appropriate constraining of $\Av_i$ leads to a banded system of equations and in consequence reduces the computational complexity of these algorithms, e.g.~\cite{mca:parekh14}.
More recent MCA works use learned dictionaries for image analysis~\cite{mca:shoham08,mca:peyre07}. Some extensions of MCA consider learning dictionaries $\Av_i$s and sparse codes jointly~\cite{mca:peyre07,mca:peyre10}. 
\begin{remark} [Learning dictionaries]
In our paper, we learn dictionaries $\Av_is$ independently. In particular, for each $i$ we minimize
\vspace{-0.05in}
\begin{align}\label{eqn: dictlearning}
\mathcal{L}_{\text{Dict}}(\Av_i) = \frac{1}{P}\sum_{p=1}^P E_{\Av_i}(\xv_i^{*,p}; \yv_i^p)
\end{align}
with respect to $\Av_i$ using SGD, where $\yv_i^p$ is the $i^{th}$ mixture component of the $p^{th}$ training sample and $\xv_i^{*,p}$ is the corresponding optimal sparse code. The columns are constrained to have unit norm. The sparse codes in each iteration are obtained with FISTA. 
\end{remark}

\section{From iterative to predictive SC and MCA}
\label{section: admm/salsa}
\vspace{-0.25cm}
\subsection{Split Augmented Lagrangian Shrinkage Algorithm (SALSA)}
\vspace{-0.25cm}
The objective functions used in SC (Equation~\ref{lasso}) and MCA (Equation~\ref{eqn: mca}) are each convex with respect to $\xv$, allowing a wide variety of optimization algorithms with well-studied convergence results to be applied~\cite{cvxOpt:combettes11}. Here we describe a popular algorithm that is general enough to solve both problems called SALSA~\cite{salsa:Afonso09}, which is an instance of ADMM. ADMM~\cite{cvxOpt:boyd11} addresses an optimization problem with the form
\vspace{-0.25cm}
\begin{align}\label{prob: optGen}
\min_{\xv} f_1(\xv) + f_2(\xv)
\end{align}
by re-casting it as the equivalent, constrained problem

\vspace{-0.15in}
\begin{align}\label{prob: admmGen}
\min_{\uv,\xv} f_1(\xv) + f_2(\uv)\:\:\:
\text{such that }\: \xv=\uv. 
\end{align}
\vspace{-0.09in}
ADMM then optimizes the corresponding scaled Augmented Lagrangian,

\vspace{-0.12in}
\begin{align}\label{eqn:augLag}
\mathcal{L}_A= f_1(\xv) + f_2(\uv)+\frac\mu2\norm{\uv-\xv-\dv}_2^2 - \frac\mu2\norm{d}_2^2,
\end{align}
where $\dv$ correspond to Lagrangian multipliers, one variable at a time until convergence.

SALSA, proposed in~\cite{salsa:Afonso09}, addresses an instance of the general optimization problem from Equation~\ref{prob: admmGen} for which convergence has been proved in~\cite{salsa:Erkstein90}. Namely, SALSA requires that (1) $f_1$ is a least-squares term, and (2) the proximity operator of $f_2$ can be computed exactly. For our most general cost function in Eqn \ref{eqn: mca}, requirement (1) is clearly satisfied, and our $f_2$ is the weighted sum of $\ell_1$ norms. In Supplemental Section~\ref{sec:suppL1}, we show that the the proximity operator of $f_2$ reduces to element-wise soft thresholding for each component, which in scalar form is given by
\begin{equation}\label{eqn: shrink}
    \soft(z;\alpha) =
        \begin{cases}
            z-\alpha, &z>\alpha\\
            0,        &|z|\leq\alpha\\
            z+\alpha, &z<-\alpha
        \end{cases}.
\end{equation}
When applied to a vector, $\soft(\zv;\alpha)$ performs soft thresholding element-wise. Thus, SALSA is guaranteed to converge for the multiple-dictionary sparse coding problem.


\begin{algorithm}[t]
	\caption{SALSA (Single Dictionary~\cite{salsa:Afonso09})}
	\label{alg: salsa1}
	\begin{algorithmic}[1]
		
		\STATE {\bfseries Input:} $\alpha\geq0, \mu>0$\\
		\hspace{0.41in}$\yv\in\RR^M, \Av\in\RR^{M\times N}$ \\
		\vspace{0.06in}
		\STATE {\bfseries Initialize:} $\xv = \Av\Tt\yv$ and $\dv=0$\\
		\vspace{0.06in}
		\REPEAT
		\STATE $\uv = \soft(\xv+ \dv;\alpha/\mu)$
		\STATE Solve for $\xv$:
		$\hspace{0.2in}\left[\mu \bf{I}+ \Av^T\Av\right]\xv= \Av\Tt\yv + \mu(\uv - \dv)$
		\STATE $\dv =  \dv - \uv + \xv$
		\UNTIL{change in $\xv$ below a threshold}
		
	\end{algorithmic}
\end{algorithm}
\begin{algorithm}[t]
	\caption{SALSA (Two Dictionaries~\cite{salsa:ivan})}
	\label{alg: salsa2}
	\begin{algorithmic}[1]
		\STATE {\bfseries Input:} $\alpha_1,\alpha_2\geq0, \mu >0$\\
		\hspace{0.41in}$\yv\in\RR^M,\Av\in\RR^{M\times (N_1+N_2)}$ 
		\STATE {\bfseries Initialize:} $\xv = \Av\Tt\yv, \dv \coloneqq \vecTwo{\dv_1}{\dv_2} = 0$\\
		\REPEAT
		\STATE $\uv 
		= \vecTwo{\soft(\xv_1+\dv_1; \alpha_1 / \mu)}{\soft(\xv_2+\dv_2; \alpha_2 / \mu)}$
		\STATE Solve for $\xv$:\\
		$\hspace{0.2in}\left[\mu \bf{I}+ \Av\Tt\Av\right]\xv= \Av\Tt\yv + \mu(\uv - \dv)$
		\STATE $\dv = \dv - \uv + \xv $
		\UNTIL{change in $\xv$ below a threshold}		
	\end{algorithmic}
\end{algorithm}


SALSA is given in Algorithms~\ref{alg: salsa1} and~\ref{alg: salsa2} for the single-dictionary case and the MCA case involving two dictionaries\footnote{In this paper we consider the MCA framework with two dictionaries. Extensions to more than two dictionaries are straightforward.}, respectively. Note that in Algorithm \ref{alg: salsa2}, the $\uv$ and $\dv$ updates can be performed with element-wise operations.
The $\xv$-update, however, is non-separable with respect to components $\{\xv_i\}_{i=1}^D$ for general $\Av$; the system of equations in the $\xv$-update cannot be broken down into $D$ sub-problems, one for each component (in contrast, 1st order methods such as FISTA update components independently).
We call this the \textit{splitting step}. 

As mentioned in Section \ref{section: background}, the $\xv$-update is often simplified to element-wise operations by constraining matrix $\Av$ to have special properties. For example: requiring $\Av\Av\Tt=\rho \bf{I}$, $\rho\in\RR_+$, reduces the $\xv$-update step to element-wise division (after applying the matrix inverse lemma). In~\cite{unrolled:yang16}, $\Av$ is set to be the partial Fourier transform, reducing the system of equations of the $\xv$-update to be a series of convolutions and element-wise operations. In our work, as is typical in the case of SC, $\Av$ is a learned dictionary without any imposed structure. 

\begin{figure}[t]
	\centering
	\includegraphics[width = 0.7\textwidth]{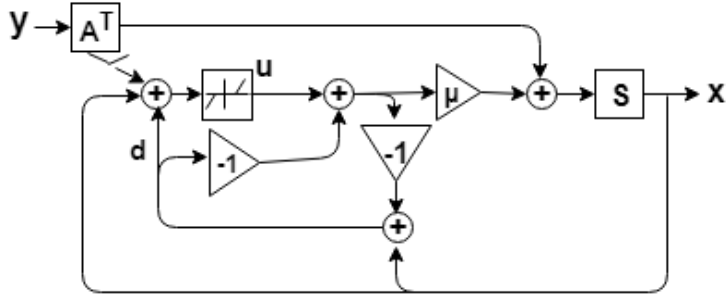}
	\caption{A block diagram of SALSA. The one-time initialization $\xv = \Av\Tt\yv$ is represented by a gate on the left.}
	\label{fig: salsaDiagram}
\end{figure}

Note that one way to solve for $\xv$ in Algorithms~\ref{alg: salsa1} and~\ref{alg: salsa2} is to compute the inverse of regularized Hessian matrix $\mu I + \Av\Tt\Av$. This however needs to be done just once, at the very beginning, as this matrix remains fixed during the entire run of SALSA. We abbreviate the inverted matrix as
\begin{align}
\label{eqn:sDef}
    \Sv = (\mu \bf{I} + \Av\Tt\Av)^{-1}.
\end{align}
We call this matrix a \textit{splitting operator}. Note that the inversion process couples together the dictionary elements (and hence also the dictionaries) in a non-linear fashion. This is an advanced utilization of prior knowledge not seen in the comparator methods of Section \ref{section: results}. The recursive block diagram of SALSA is depicted in Figure \ref{fig: salsaDiagram}.

\vspace{-0.4cm}
\subsection{Learned SALSA (LSALSA)} \label{sec:LSALSA}
\begin{figure*}[!ht]
	\centering
	\includegraphics[width = \textwidth]{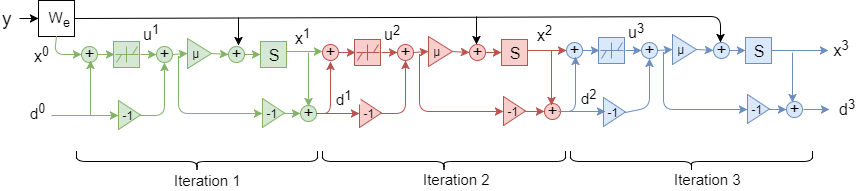}
	\caption{The deep learning architecture of LSALSA for $T=3$. The soft-thresholding function, defined in Equation~\ref{eqn: shrink}, is an activation function found in each layer of the network and at the end.}
	\label{fig: lsalsaUnfolded}
\end{figure*}
\vspace{-0.5cm}We now describe our proposed deep encoder architecture that we refer to as Learned SALSA (LSALSA). Consider truncating the SALSA algorithm to a fixed number of iterations $T$ and then time-unfolding it into a deep neural network architecture that matches the truncated SALSA's output exactly. The obtained architecture is illustrated in Figure~\ref{fig: lsalsaUnfolded} for $T=3$, and the formulas for the $t^{th}$ layer w.r.t. the $(t-1)^{th}$ iterates are described via pseudocode in Algorithms \ref{alg: fprop1} and Algorithm \ref{alg: fprop2} for the single-dictionary and MCA cases, respectively. Note that Algorithms~\ref{alg: salsa2} and~\ref{alg: fprop2} are the most general algorithms considered by us whereas Algorithms~\ref{alg: salsa1} and~\ref{alg: fprop1} are their special, i.e. single-dictionary, cases. 

The LSALSA model has two matrices of learnable parameters: $\Sv$ and $\Wve$. We initialize these to achieve an exact correspondence with SALSA:
\vspace{-0.02in}
\begin{align}\label{eqn:initializations}
    \Wve = \Av\Tt \in\RR^{N\times M}\:\:\:\text{and}\:\:\:\Sv  = \left(\mu \bf{I} + \Av\Tt\Av\right)^{-1} \in\RR^{N\times N},
\end{align}
where $N=N_1+N_2$ in the MCA case. All splitting operators $\Sv$ share parameters across the network.
LSALSA's two matrices of parameters can be trained with standard backpropagation. Let $\xv = f_e(\Wv_e,\Sv,\yv)$ denote the output of the LSALSA architecture after a forward propagation of $\yv$. The cost function used for training the model is defined as
\vspace{-0.03in}
\begin{align}
\label{eqn: learningCostFcn}
    \mathcal{L}(\Wv_e,\Sv) = \frac{1}{2P}\sum_{p=1}^P\norm{\xv^{*,p}-f_e(\Wv_e,\Sv,\yv^p)}_2^2.
\end{align} 

\begin{algorithm}[t]
	\caption{LSALSA  Forward Pass (Single Dictionary)}
	\begin{algorithmic}
		\STATE{\bfseries Input:} $\alpha\geq0, \mu > 0,\yv\in \mathbb{R}^M$\\
		\hspace{0.41in}$\xv(0) = \Wve\yv$ , $\dv(0)=0 $\\
		\vspace{0.06in}
		\FOR{$t=1$ {\bfseries to} $T $}
		\STATE $\uv(t) = \soft(\xv(t-1)+\dv(t-1); \alpha / \mu)$
		\STATE $\xv(t) = \Sv(\Wve\yv+\mu(\uv(t)-\dv(t-1)))$
		\STATE $\dv(t) =  \dv(t-1) - \uv(t)+\xv(t)$
		\ENDFOR  
		\vspace{0.06in}
		\STATE {\textbf{Output: } $\soft(\xv(t);\alpha/\mu)$}
	\end{algorithmic}
	\label{alg: fprop1}
\end{algorithm}

\begin{algorithm}[t]
	\caption{LSALSA  Forward Pass (Two Dictionaries):}
	\begin{algorithmic}
		\STATE{\bfseries Input:} $\alpha_1,\alpha_2\geq0, \mu > 0,\yv\in\mathbb{R}^M$\\
		\hspace{0.41in}$\xv(0) = \Wve\yv$ , $\dv(0)=0 $\\
		\vspace{0.06in}
		\FOR{$t=1$ {\bfseries to} $T $}
		\STATE $\uv(t) =\vecTwo{\soft(\xv_1(t-1) + \dv_1(t-1); \alpha_1 / \mu)}{\soft(\xv_2(t-1) + \dv_2(t-1); \alpha_2 / \mu)}$
		\STATE $\xv(t) =\Sv(\Wve\yv+\mu(\uv(t)-\dv(t-1)))$
		\STATE $\dv(t) = \dv(t-1) - \uv(t)+\xv(t)$
		\ENDFOR  
		\vspace{0.06in}
		\STATE {\textbf{Output: } $\vecTwo{\soft(\xv_1(t); \alpha_1 / \mu)}{\soft(\xv_2(t); \alpha_2 / \mu)}$}
	\end{algorithmic}
	\label{alg: fprop2}
\end{algorithm}

To summarize, LSALSA extends SALSA. SALSA is meant to run until convergence, where LSALSA is meant to run for $T$ iterations, where $T$ is the depth of the network. Intuitively, the backpropagation steps applied during training in LSALSA fine-tune the “splitting step” so that $T$ iterations can be sufficient to achieve good-quality sparse codes (those are obtained due to the existence of nonlinearities).
The SALSA algorithm relies on cumulative Lagrange Multiplier updates to “explain away” code components, while separating sources. This is especially important in MCA, where similar atoms from different dictionaries will compete to represent the same segment of a mixed signal.
The Lagrange Multiplier updates translate to a cross-layer connectivity pattern in the corresponding LSALSA network (see the $d$-updates in Figure \ref{fig: lsalsaUnfolded}), which has been shown to be a beneficial architectural feature in e.g.~\cite{skip:greff2016highway,skip:liao2016,skip:orhan2018}. During training, LSALSA is fine-tuning the splitting operator $\Sv$ so that it need not rely on a large number of cumulative updates.
However, we show in Section~\ref{sec:analysis} that even after training, forward propagation through an LSALSA network is equivalent to the application of a truncated ADMM algorithm applied to a new, learned cost function that generalizes the original problem.

\section{Analysis of LSALSA}
\label{sec:analysis}
\vspace{-0.25cm}
\subsection{Optimality Property for LSALSA}
\vspace{-0.25cm}
Typically, analyses of ADMM-like algorithms rely on the optimality of each primal update, e.g. that $\xv\iter{k+1}=\argmin{\xv}\mathcal{L}_A(\xv,\uv\iter{k+1};\dv\iter{k})$~\cite{cvxOpt:boyd11,cvxOpt:fadmmGoldstein2014,ncAdmm:Wang2019}. In Theorem \ref{Thm:reparam} we show that LSALSA provides optimal primal updates with respect to a generalization of the Augmented Lagrangian (\ref{eqn:augLag}) parameterized by $\Sv$. The proof is provided in Supplemental Section~\ref{supp:derivReparam}.

\begin{theorem}[LSALSA Optimality]\label{Thm:reparam}
Given a neural network with the LSALSA architecture as described in Section \ref{sec:LSALSA}, there exists an Augmented Lagrangian for which the LSALSA network provides optimal primal updates. In particular, for learned matrices $\Sv$ and $\Wve$, we have 
\begin{align}
    \hLc= \hat{f_1}(\xv;\Sv)+\ell_1(\uv) + \frac\mu2\norm{\uv-\xv-\dv}^2-\frac\mu2\norm{\dv}^2,
\end{align}
where
\begin{align}
    \hat{f_1}(\xv;\Sv) = \frac12\xv\Tt\left[\Sv^{-1}-\mu I\right]\xv -(\Wve \yv)\Tt \xv + \frac12\yv\Tt\yv,
\end{align}
and $\ell_1(\uv)$ represents a sum of L1-terms as in (\ref{eqn: mca}).
\end{theorem}

\begin{remark}[LSALSA as an Instance of ADMM]
    Note that by plugging in the initializations of $\Sv$ and $\Wve$, given in Equations \ref{eqn:initializations}, we recover the original Augmented Lagrangian. Then, from the perspective of Theorem~\ref{Thm:reparam}, \textit{LSALSA at inference is equivalent to applying $T$ iterations of ADMM on a new, learned cost function that generalizes the original problem in Equation~\ref{eqn:augLag}.}
\end{remark}
\begin{remark}[LSALSA Provides Sparse Solutions]
    Since $\hLc$ employs the $\ell_1$-norm in the usual way and LSALSA's $\uv$-update is standard soft-thresholding, we can expect LSALSA to enforce sparsity given sufficient iterations~\cite{cvxOpt:boydbook2004,cvxOpt:l1_2014}.
\end{remark}

We show in Section~\ref{sec:admmAnalysis} that the optimal direction for $\hLc$ is related to the optimal direction for $\Lc_A$, and in Section~\ref{sec:stochAnalysis} we show that gradient descent along $\hLc$ is equivalent to a modified gradient descent along $\Lc_A.$ For simplicity, we consider the case of learned, symmetric $\Sv$ while holding fixed $\Wve\equiv\Av\Tt$.

\vspace{-0.3cm}
\subsection{Modified descent direction: deterministic framework}\label{sec:admmAnalysis}
\vspace{-0.2cm}
Though $\hLc$'s dependence on $\uv$ and $\dv$ is standard in ADMM settings~\cite{cvxOpt:boyd11}, the learned data-fidelity term $\hat{f_1}$ that commands $\xv$-directions is now a data-driven quadratic form that relies on the weight matrix $\Sv$ that parameterizes LSALSA.
We will next rewrite the new cost function in terms of the original Augmented Lagrangian:
\begin{equation}\label{eqn: LSrelate}
    \hLc(\xv,\uv,\dv) = \Lc_A(\xv,\uv,\dv) + \hat{f_1}(\xv;\Sv) - \frac12\norm{\yv-\Av\xv}^2_2.
\end{equation}
The optimality condition for $\hLc$ can be written
\begin{align*}
   0 & =\nabla_{\xv}\hLc(\xv^*, \uv, \dv)\\
    &=\nabla_{\xv}\left(\Lc_A(\xv^*, \uv, \dv) +\hat{f_1}(\xv;\Sv) - \frac12\norm{\yv-\Av\xv}^2_2\right)\\
    &=\nabla_{\xv}\Lc_A(\xv^*, \uv, \dv) +\left[\Sv^{-1}-\mu I - \Av\Tt\Av \right]\xv^*.
\end{align*}
Then, using $\nabla_{\xv}^2\Lc_A=\mu I + \Av\Tt\Av$ we can write the LSALSA update as
\begin{align}
\label{eqn: rootFind}
    0&=\nabla_{\xv}\Lc_A(\xv^*, \uv, \dv) + \left[\Sv^{-1} -\nabla_{\xv}^2\Lc_A \right]\xv^*\\
\label{eqn:learnedOptimality}
\Rightarrow    &\left[\Sv^{-1} - \nabla_{\xv}^2\Lc_A \right]\xv^* = -\nabla_{\xv}\Lc_A(\xv^*, \uv, \dv).
\end{align}
The root-finding problem posed in (\ref{eqn: rootFind}) and equivalent system of equations in  (\ref{eqn:learnedOptimality}) resemble a Newton-like update, but using a learned modification of the original Lagrangian's Hessian matrix.
Note that at initialization (using Formula~\ref{eqn:initializations}), the left-hand-side cancels to zero, recovering the optimality condition for the original problem. 
This also admits an intuition that LSALSA is incorporating prior knowledge, learned from the training data, that could be made to balance between optimality of the original problem while maintaining some relationship with the training data distribution.

\vspace{-0.25cm}
\subsection{{Modified descent direction: stochastic framework}}\label{sec:stochAnalysis}
\vspace{-0.25cm}
We will next look at (L)SALSA through the prism of worst-case analysis, i.e. by replacing the optimal primal steps with stochastic gradient descent. This effectively enables us to analyze (L)SALSA as a stochastic alternated optimization approach solving a general saddle point problem, and we show that LSALSA leads to faster convergence under certain assumptions that we stipulate. Our analysis is a direct extension of that in~\cite{sao:beyondBackprop}. We provide the final statement of the theorem below and defer all proofs to the supplement.

\vspace{-0.25cm}
\subsubsection{Problem formulation}
Consider the following general saddle-point problem:
\begin{align}
    \max_{\phi_1,...,\phi_{K_2}}\min_{\theta_1,...,\theta_{K_1}} &\Lu{}(\theta_1,...,\theta_{K_1};\phi_1,...,\phi_{K_2})\\
        &\Updownarrow \nonumber\\
    \max_{\phiv}\min_{\thetav}\ &\Lu{}(\thetav;\phiv),
    \label{eq:sadp}
\end{align}
using $\thetav = [\theta_1,...,\theta_{K_1}]$ to denote the collection of variables to be minimized, and $\phiv = [\phi_1,...,\phi_{K_2}]$ the variables to be maximized. We denote the entire collection of variables as $\xv=[\thetav, \phiv]\in\mathbb{R}^{K},$ where $K=K_1+K_2$ is the total number of arguments. We denote with $x_d$ the $d^{th}$ entry in $\xv$. For theoretical analysis we consider a smooth function $\Lu{}$ as is often done in the literature (especially for $\ell_1$ problems, as discussed in~\cite{l1smooth:LangeZHV14,l1smooth:schmidt2007}).

Let $(x_1^*,...,x_K^*)$ be the optimal solution of the saddle point problem in (\ref{eq:sadp}), where $\Lu{}$ is computed over global data population (i.e. averaged over an infinite number of samples). For each variable $x_d$, we assume a lower bound on the radii of convergence $r_d>0$. Let $\nabla_d^1 \Lu{}$ denote the gradient of $\Lu{}$ with respect to the $d^{th}$ argument evaluated on a single data sample (stochastic gradient), and $\nabla_d \Lu{}$ to be that with respect to the global data population (i.e. an ``oracle gradient''). 

We analyze an Alternating Optimization algorithm that, at the $d^{th}$ step, optimizes $\Lu{}$ with respect to $x_d$ while holding all other $x_{i\neq d}$ fixed:
\begin{equation}
\label{eqn:propIt}
    x_d^{t+1} = \Pi_d\left(x_d^t \pm \eta^t\nabla_d^1\Lu{x_d}^t\right),
\end{equation}
using the $\pm$ symbol to denote gradient \textit{descent} for $d\leq K_1$ and gradient \textit{ascent} for $d>K_1$. 
$\Pi_d$ is the projection onto the Euclidean-ball $B_2(\frac{r_d}{2},x_d^*),$ with radius $\frac{r_d}{2}$ and centered around the optimal value $x_d^*$: this ensures that for each $d$, all iterates of $x_d$ remain within the $r_d$-ball around $x_d^*$\footnote{this assumption can be potentially eliminated with carefully selected initial stepsizes.}.

\vspace{-0.25cm}
\subsubsection{Assumptions}\label{sec:assumps}
\vspace{-0.25cm}
The following assumptions are necessary for the Theorems in Section~\ref{sec:theoRez}. The mathematical definitions of strong-convexity, strong-concavity, and smoothness follow the standards from~\cite{nesterov2013introductory}.
\begin{assumption}[Convex-Concave]
For each $d\leq K_1$, $\Lu{x_d}^*$ is $\beta_d$-convex, and for each $d>K_1$, $\Lu{x_d}^*$ is $\beta_d$-concave within a ball around the solution $x_d*$ of radius $r_d$. 
\end{assumption}



\begin{assumption}[Smoothness]
For all $d\in\{1,...,K\}$, the function $\Lu{x_d}^*$ is $\alpha_d$-smooth.
\end{assumption}


In summary, for every $d=1,...,K$, $\Lu{x_d}^*$ is either $\beta_d$-convex or concave in a neighborhood around the optimal point, and $\alpha_d$-smooth.
Next we assume two standard properties on the gradient of the cost function.

\begin{assumption}[Gradient Stability $GS(\gamma_d)$]\label{asspn:GS}
We assume that for each $d=1,...,K,$ the following gradient stability condition holds for $\gamma_d\geq0$ over the Euclidean ball $x_d\in B_2(r_d,x_d^*)$:
\begin{equation}
    \norm{\nabla_d\Lu{x_d}^* - \nabla_d\Lu{x_d}} \leq \gamma_d\sum_{i\neq d}\norm{x_i-x_i^*}.
\end{equation}
\end{assumption}

\begin{assumption}[Assumption A.6: Bounded Gradient]\label{asspn:BG}
We assume that the expected value of the gradient of our objective function $\Lc$ is bounded by $\sigma = \sqrt{\sum_{d=1}^K \sigma_d^2}$, where:
\begin{equation}
\label{gradBound}
    \sigma_d = \sup\left\{\E{\norm{\nabla_d\Lu{x_d}}^2}: x_d\in B_2(r_d,x_d^*),\ \forall d=1,...,K\right\}.
\end{equation}
\end{assumption}

\subsubsection{Convergence statement}\label{sec:theoRez}
Denote with $\Delta_d^t=x_d^t-x_d^*$ the error of the $t^{th}$ estimate of $d^{th}$ element of the global optimizer $\xv^*$. Define the following:
\begin{equation}
    \mathcal{E}_{\textsf{SALSA}}(\beta)=\left(\frac{2}{t+3}\right)^{\frac32}\E{\sum_{d=1}^K\norm{\Delta_d^0}^2} + \frac{9\sigma^2}{[2\xi(\beta)-\gamma(2K-1)]^2(t+3)},
\end{equation}
where $\xi(\beta)$ increases monotonically with increasing $\beta.$

\begin{theorem}[Convergence of SALSA and LSALSA]\label{thm:chgCvxMod}
    Suppose that cost functions underlying SALSA $\Lc_A$ and LSALSA $\hLc$ satisfy the Assumptions in Section~\ref{sec:assumps} with convexity modulii $\beta$ and $\hat{\beta}$ (the latter is implicitly learned from the data). Assume also that the
    deep model representing LSALSA had enough capacity to learn $\hat{\beta}$ such that $\hat{\beta}>\beta,$ while keeping the same location of the global optimal fixed point, $\xv^*$\footnote{LSALSA is trained to keep the same global fixed point, see Equation~\ref{eqn: learningCostFcn}.}.
    
    Then, using the Stochastic Alternating Optimization scheme in Equation~\ref{eqn:propIt} on $\Lc_A$ and $\hLc$  such that the requirements from Theorem~\ref{thm:mainResult} are satisfied, starting from the same initial point, the error satisfies the following:\\
    for SALSA:
    \begin{equation}
        \sum_{d=1}^K\norm{\Delta_d^{t+1}}^2 \leq
            \mathcal{E}_{\textsf{SALSA}}(\beta),
    \end{equation}
    and for LSALSA:
    \begin{equation}
        \sum_{d=1}^K\norm{\Delta_d^{t+1}}^2 \leq
            \mathcal{E}_{\textsf{LSALSA}}(\hat{\beta}) = \mathcal{E}_{\textsf{SALSA}}(\beta) - \Delta_{\beta},
    \end{equation}
    where
    \begin{equation}
        \Delta_{\beta} = \mathcal{O}
            \left(\frac{\hat{\beta}^2 - \beta^2}{(2\beta\hat{\beta})^2}\right).
    \end{equation}
\end{theorem}

The above theorem states that, given enough capacity of the deep model, LSALSA can learn steeper descent direction than SALSA. We provide below an intuition for that. Consider the gradient descent step (or its stochastic approximation) for $\hLc$ in the $\xv$-direction as given below
\begin{align}
    \xv\iter{k+1}&=\xv\iter{k}-\eta^k\nabla_{\xv}\hLc(\xv\iter{k}, \uv\iter{k+1}, \dv\iter{k}) \nonumber\\
    &=\xv\iter{k}-\eta^k\nabla_{\xv}\left(\Lc_A^k
        +\phi(\xv;\Sv) - \frac12\norm{\yv-\Av\xv}^2_2\right) \nonumber\\
    &=\xv\iter{k}-\eta^k\nabla_{\xv}\Lc_A^k 
        -\eta^k\left[\Sv^{-1}-\mu I - \Av\Tt\Av \right]\xv\iter{k} \nonumber\\
    &=\underbrace{\xv\iter{k}-\eta^k\nabla_{\xv}\Lc_A^k}_{\text{unlearned descent step}}
        -\eta^k\left[\Sv^{-1} -\nabla_{\xv}^2\Lc_A \right]\xv\iter{k} \nonumber\\
\label{eqn:leakyGradDescent}
    &=\left[I- \eta^k P\right]\xv\iter{k}-\eta^k\nabla_{\xv}\Lc_A^k,
\end{align}
where $P:=\Sv^{-1} -\nabla_{\xv}^2\Lc_A$.

This update can be seen as taking first a gradient descent step and then pushing the optimizer further in the learned direction, which we empirically show is a faster direction of decent

\section{Numerical Experiments} 
\label{section: results} 
We now present a variety of sparse coding inference tasks to evaluate our algorithm's speed, accuracy, and sparsity trade-offs. For each task (including both SC and MCA), we consider a variety of settings of $T$, i.e. the number of iterations, and do a full hyperparameter grid search for each setting. In other words, we ask ``how well can each encoding algorithm approximate the optimal codes, given a fixed number of stages?''. We compare LSALSA, truncated SALSA, truncated FISTA, and LISTA~\cite{DBLP:conf/icml/GregorL10} in terms of their RMSE proximity to optimal codes, sparsity levels, and performance on classification tasks. Both LSALSA and LISTA are implemented as feedforward neural networks. For MCA experiments, we run FISTA and LISTA using the concatenated dictionary $\Av$.

We focus on the inference problem and thus learn the dictionaries off-line as described in Section \ref{section: background}. Dictionary learning is performed only once for each data set, and the resulting dictionaries are held constant across all methods and experiments herein (visualization of the atoms of the obtained dictionaries can be found in Section~\ref{sec:B} in the Supplement). For MCA, the independently-learned dictionaries are still used, creating difficult ill-conditioned problems (because each dictionary is able to at least partially represent both components).

To train the encoders, we minimize Equation \ref{eqn: learningCostFcn} with respect to $\Wve$ and $\Sv$ using vanilla Stochastic Gradient Descent (SGD). We considered the optimization complete after a fixed number of epochs, or when the relative change in cost function fell below a threshold of $10^{-6}$. During hyperparameter grid searches, only 10 epochs through the training data were allowed; for testing, 100 epochs of training were allowed (usually the tolerance was reached before 100 epochs). 
The optimal codes are determined prior to training by solving the convex inference problem with fixed $\alpha^*$ and $\mu^*$, e.g. by running FISTA or SALSA to convergence (details are discussed in each section). In order to set the $\alpha^*,\mu^*$, we fix $\mu^*=10$ and tune $\alpha^*$ to yield an average sparsity of at least 89\%.
We then slowly increase $\alpha*$s until just before the optimal sparse codes' fail to provide recognizable image reconstructions. We take the simplest approach to image reconstruction: simply multiplying the sparse code with its corresponding dictionary. No additional learning was performed to achieve reconstruction: i.e. for LSALSA we have $\Av_i\cdot (f_e(\Wv_e,\Sv,\yv))_i$, where $f_e(\Wv_e,\Sv,\yv))_i$ represents the $i-$th component of the encoder's output. 

We implemented the experiments in Lua using Torch7, and executed the experiments on a 64-bit Linux machine with 32GB RAM, i7-6850K CPU at 3.6GHz, and GTX 1080 8GB GPU. The hyperparameters were selected via a grid search with specific values listed in the Supplement, Section~\ref{sec:AB}.

\subsection{Single Dictionary (SC) Case}
\vspace{-0.25cm}
\label{subsec:sdc}
We run SC experiments with four data sets: Fashion MNIST~\cite{data:fmnistxiao2017} ($10$ classes), ASIRRA~\cite{data:asirra} ($2$ classes), MNIST~\cite{data:mnist} ($10$ classes), and CIFAR-10 ~\cite{data:cifar} ($10$ classes). The ASIRRA data set is a collection of natural images of cats and dogs. We use a subset of the whole data set: $4000$ training images and $1000$ testing images as commonly done~\cite{data:catsdogsPAPER}. The results for MNIST and CIFAR-10 are reported in Section~\ref{sec:C} in the Supplement. 

The $32\times32$ Fashion MNIST images were first divided into $10\times10$ non-overlapping patches (ignoring extra pixels on two edges), resulting in $9$ patches per image. Then, optimal codes were computed for each vectorized patch by minimizing the objective from Equation~\ref{lasso} with FISTA for $200$ iterations. The ASIRRA images come in varying sizes. We resized them to the resolution of $224\times224$ via Torch7's bilinear interpolation and converted each image to grayscale. Then we divided them into $16\times16$ non-overlapping patches, resulting in $196$ patches per image. Optimal codes were computed patch-wise as for Fashion MNIST, but taking $700$ iterations to ensure convergence on this more difficult SC problem. For Fashion MNIST we selected $\alpha^*=0.15$ and for ASIRRA, $\alpha^*=0.5.$ using criteria mentioned earlier in the Section.

The data sets were then separated into training and testing sets.  The training patches were used to produce the dictionaries. Visualizations of the dictionary atoms are provided in Section~\ref{sec:B} in the Supplement. An exhaustive hyper-parameter search\footnote{The parameter settings that we explored in all our experiments are provided in the Supplement.} was performed for each encoding method and for each number of iterations $T$, to minimize RMSE between obtained and optimal codes. The hyper-parameter search included $\alpha$ for all methods, $\mu$ for SALSA and LSALSA, as well as SGD learning rates and learning rate decay schedules for LSALSA and LISTA training.

The obtained encoders were used to compute sparse codes on the test set. Those were then compared with the optimal codes via RMSE. The results for Fashion MNIST are shown both in terms of the number of iterations and the wallclock time in seconds used to make the prediction (Figure~\ref{fig: fmnistErr_and_time}). It takes FISTA more than $15$ iterations and SALSA more than $5$ to reach the error achieved by LSALSA in just one. Near $T=100$, both FISTA and SALSA are finally converging to the optimal codes. LISTA outperforms FISTA at first, but does not show much improvement after $T>10$. Similar results for ASIRRA are shown in the same figure. On this more difficult problem, it takes FISTA more than $50$ iterations and SALSA more than $20$ to catch up with LSALSA with a single iteration. LISTA and LSALSA are comparable for $T\leq5$, after which LSALSA dramatically improves its optimal code prediction and, similarly as in case of Fashion MNIST, shows advantage in terms of the number of iterations, inference time, and the quality of the recovered sparse codes over other methods. 

\begin{figure}[!ht]
	\centering
	\begin{tabular}{cc}
		\hspace{-0.14in}\includegraphics[width = 0.5\textwidth]{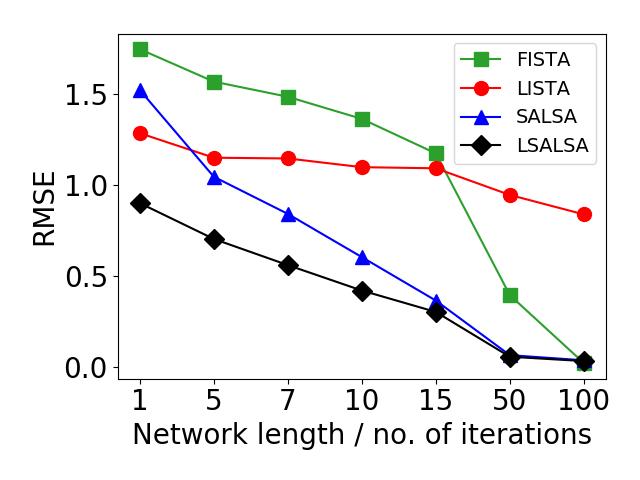}&
		\hspace{-0.21in}\includegraphics[width = 0.5\textwidth]{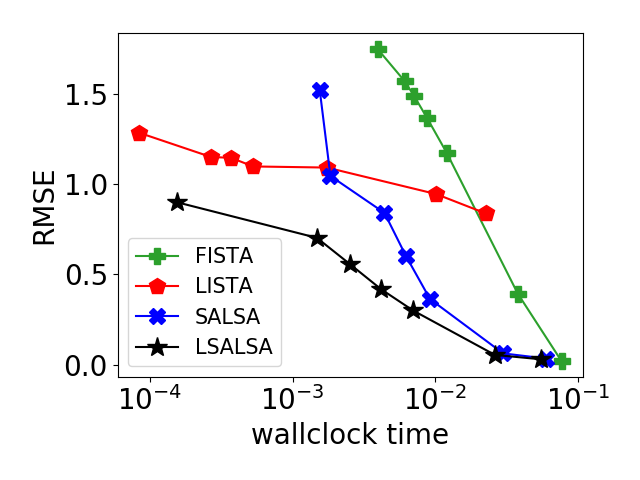}\\
		a) & b) \\
		\hspace{-0.18in}\includegraphics[width = 0.5\textwidth]{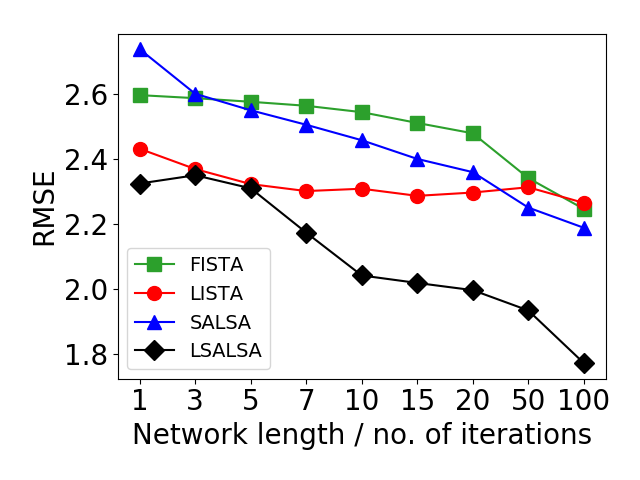}&
		\hspace{-0.26in}\includegraphics[width = 0.5\textwidth]{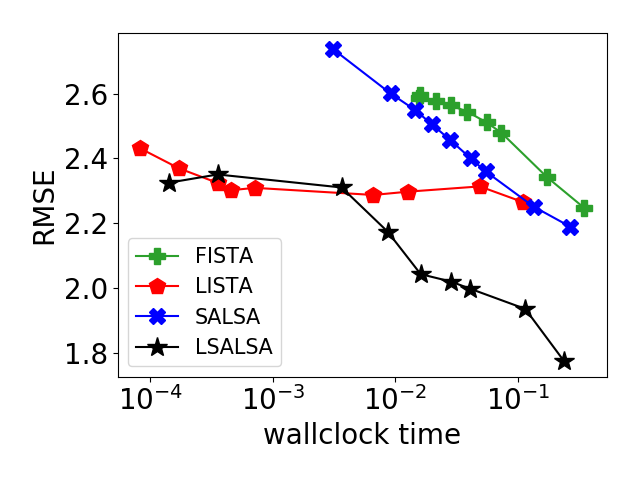}\\ c) & d)\\
	\end{tabular}
	\vspace{-0.1in}
	\caption{Code prediction error as a function (\textbf{a}) iteration count, and (\textbf{b}) inference wallclock time for Fashion MNIST (\textbf{a,b}) and ASIRRA (\textbf{c,d}).}
	\label{fig: fmnistErr_and_time} 
	\vspace{-0.1in}
\end{figure}

We also investigated which method yields better codes in terms of a classification task. We trained a logistic regression classifier to predict the label from the corresponding optimal sparse code, then ask: ``can the classifier still recognize a fast encoder's estimate to the optimal code?''. For Fashion MNIST each image is associated with $9$ optimal codes (one for each patch), yielding a total feature length of $9\times10\times10=900$. The Fashion MNIST classifier was trained until it achieved $0\%$ classification error on the optimal codes. For ASIRRA, each concatenated optimal code had length $196\times16\times16=50176$; to reduce the dimensionality we applied a random Gaussian projection $\mathcal{G}:\RR^{50176}\rightarrow\RR^{500}$ before inputting the codes into the classifier. The classifier was trained on the optimal projected codes of length $500$ until it achieved $0.5\%$ error. The results for Fashion MNIST and ASIRRA are shown in Table~\ref{tab: Fmnist1d} and~\ref{tab: catdog1d}, respectively, in Section~\ref{sec:C} in the Supplement.
\textit{Note}: The classifier was trained on the target test codes so that the resulting classification error is only due to the difference between the optimal and estimated codes.
In conclusion, although the FISTA, LISTA, or SALSA codes may not look that much worse than LSALSA in terms of RMSE, we see in the Tables that the expert classifiers cannot recognize the extracted codes, despite being trained to recognize the optimal codes which the algorithms seek to approximate.

\vspace{-0.35cm}
\subsection{MCA: Two-Dictionary Case}
\label{subsec:mca}
\vspace{-0.25cm}
\subsubsection{Data Preparation}
\vspace{-0.25cm}
We now describe the dataset that we curated for the MCA experiments. We address the problem of decoupling numerals (text) from natural images, a topic closely related to text detection in natural scenes~\cite{expSet:mcaText2017,expSet:tian2015text,expSet:uyghur2018}. Following the notation introduced previously in the paper, we set $\yv_1^p$s to be the whole $32\times32$ MNIST images and $\yv_2^p$s to be non-overlapping $32\times32$ patches from ASIRRA (thus we have $49$ patches per image). We obtain $196$k training and $49$k testing patches from ASIRRA, and $60$k training and $10$k testing images from MNIST. We add together randomly selected MNIST images and ASIRRA patches to generate $588$k mixed training images and $49$k mixed testing images. Optimal codes were computed using SALSA (Algorithm \ref{alg: salsa2}) for $100$ iterations, ensuring that each component had a sparsity level greater than $89\%$, while retaining visually recognizable reconstructions. The values selected were $\alpha_1=0.125^*,$ $\alpha_2^*=0.2$, $\mu^*=10$. We also performed MCA experiments on additive mixtures of CIFAR-10 and MNIST images. Those results can be found in Section~\ref{sec:D} in the Supplement.

\subsubsection{Results}
\vspace{-0.25cm}
An exhaustive hyper-parameter search was performed for each encoding method and each number of iterations $T$. The hyper-parameters search included $\alpha$ for FISTA and LISTA, $\alpha_1,\alpha_2,\mu$ for SALSA and LSALSA, as well as SGD learning rates for LSALSA and LISTA training. The code prediction error curves are presented in Figure~\ref{fig: mca2Err_and_Time}. LSALSA steadily outperforms the others, until SALSA catches up around $T=50$. FISTA and LISTA, without a mechanism for distinguishing two dictionaries, struggle to estimate the optimal codes.

\begin{figure}[ht!]
	\centering
	\begin{tabular}{cc}
		\hspace{-0.08in}\includegraphics[width=0.45\textwidth]{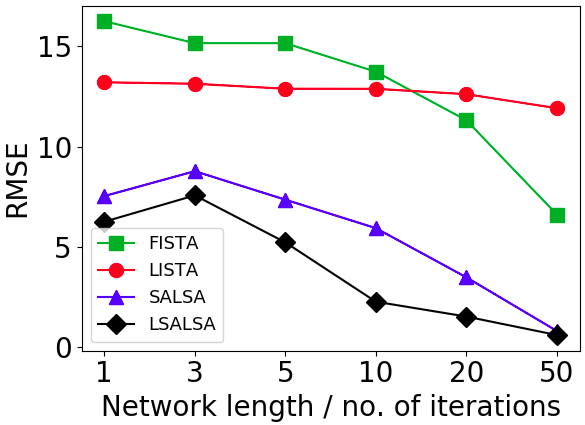}&
		\hspace{-0.07in}\includegraphics[width=0.45\textwidth]{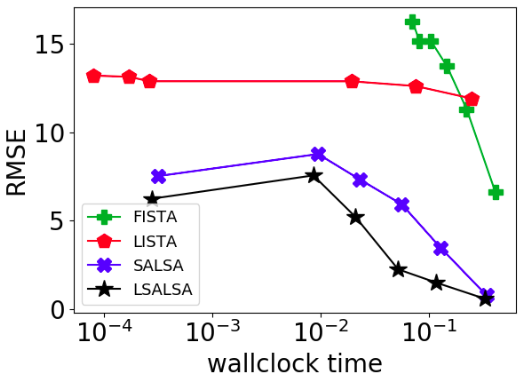}
	\end{tabular}
	\vspace{-0.1in}
	\caption{ MCA experiment using MNIST + ASIRRA data set. (\textbf{left}) Code prediction errors for varying numbers of iterations. (\textbf{right}) Code prediction error versus inference wallclock time.}
	\label{fig: mca2Err_and_Time}
\end{figure}

\begin{figure}[ht!]
\vspace{-0.15in}
	\centering
	\begin{tabular}{cc}
		\hspace{-0.08in}\includegraphics[width=0.45\textwidth]{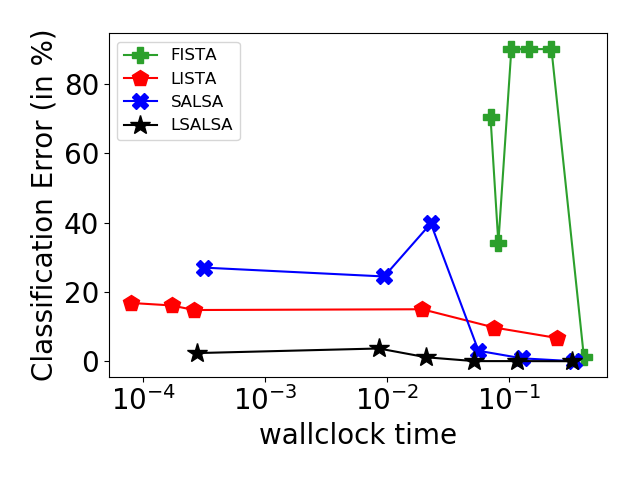}&
		\hspace{-0.07in}\includegraphics[width=0.45\textwidth]{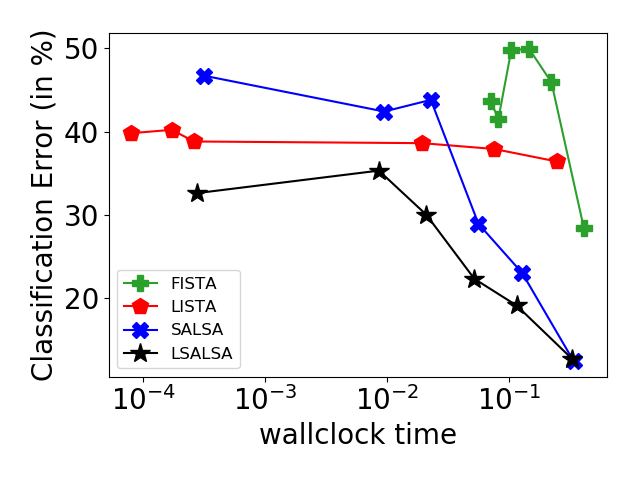}
	\end{tabular}
	\vspace{-0.2in}
	\caption{ MCA experiment separating MNIST + ASSIRA components: The trade-off between the sparse codes classification error Vs their inference time  is captured for different network lengths on (left) for MNIST (right) for ASSIRA.}
	\label{fig: cd_mnist_clFigure }
\end{figure}

\begin{figure}[ht!]\centering
\vspace{-0.15in}
	\hspace{-0.01in}\subfigure[$T=1$]{\includegraphics[width=0.33\textwidth]{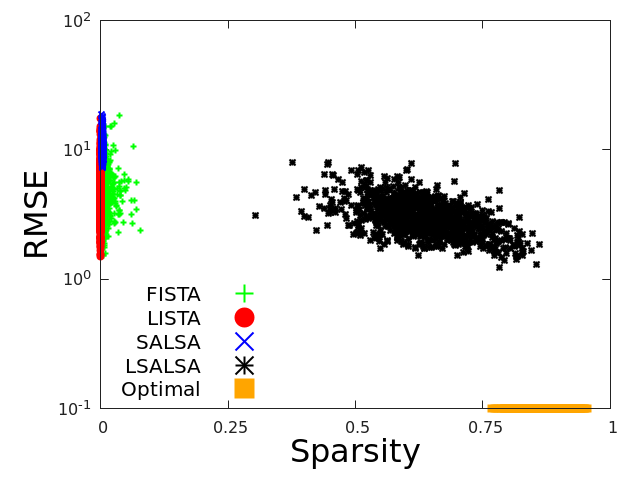}}
	\subfigure[$T=3$]{\includegraphics[width=0.33\textwidth]{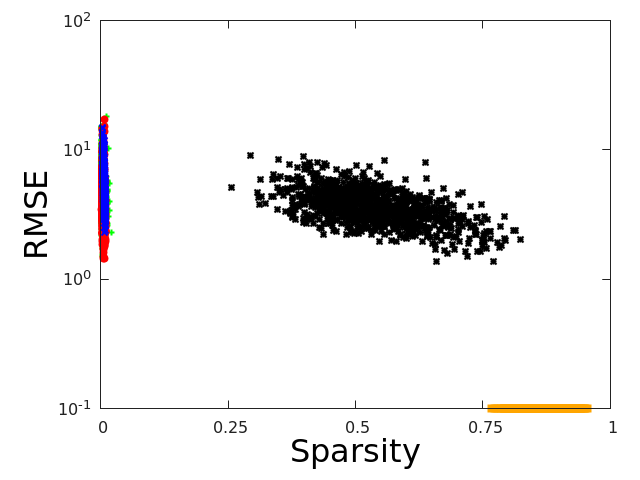}}\hspace{0.1em}%
	\subfigure[$T=5$]{\includegraphics[width=0.33\textwidth]{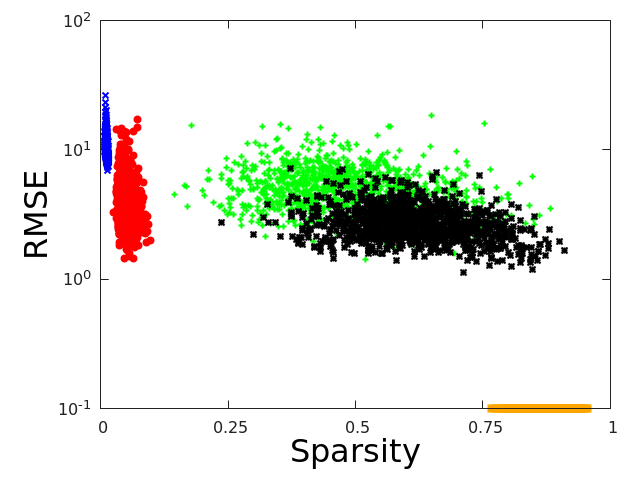}} \hspace*{\fill}%
    \hspace{-0.07in}\subfigure[$T=10$]{\includegraphics[width=0.33\textwidth]{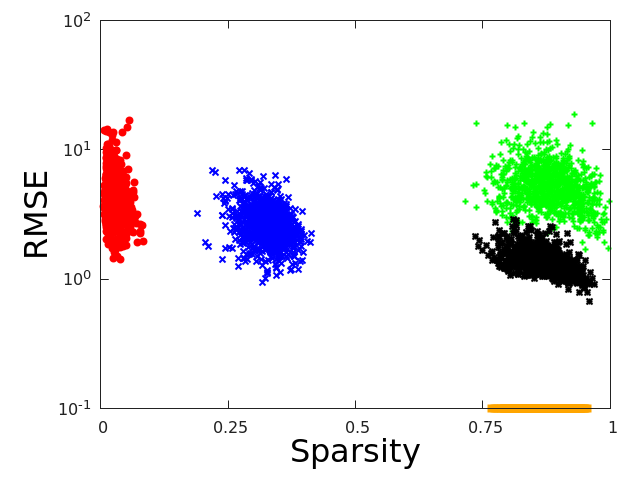}}
	\subfigure[$T=20$]{\includegraphics[width=0.33\textwidth]{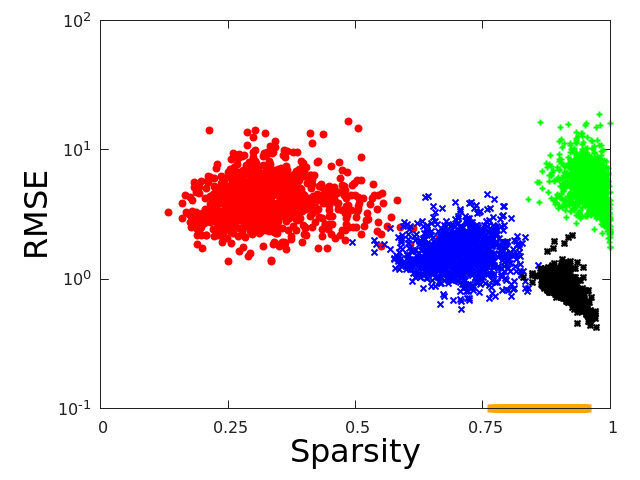}}\hspace{0.1em}%
    \vspace{-0.1in}
	\subfigure[$T=50$]{\includegraphics[width=0.33\textwidth]{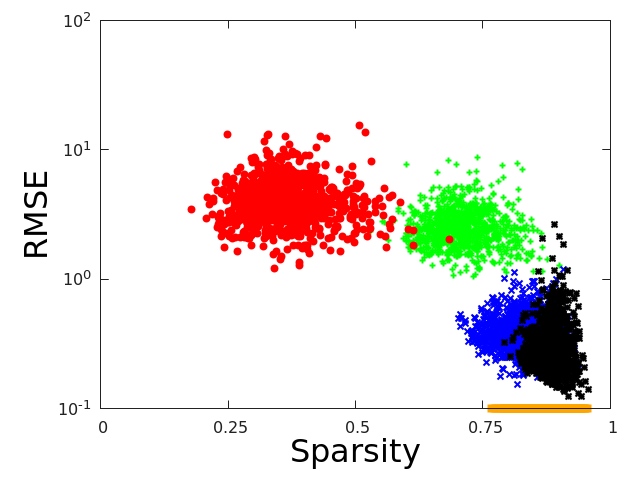}}
	\caption{\label{fig: cdSparsity_multiN} Sparsity/accuracy trade-off analysis for ASIRRA obtained for the source separation experiment with MNIST + ASIRRA data set. Each method corresponds to a colored point cloud, where each point corresponds to one sample from the ASIRRA test data set. LSALSA (black) achieves the higher sparsity and/or lower code estimation error than the other methods for each $T$.}
\end{figure}

In Figure~\ref{fig: cdSparsity_multiN} we illustrate each method's sparsity/accuracy trade-off on the ASIRRA test data set, while varying $T$ (Supplemental Section~\ref{sec:E} contains a similar plot for MNIST). For each data point in the test set, we plot its sparsity vs. RMSE code-error, resulting in a point-cloud for each algorithm. For example, a sparsity value of 0.6 corresponds to 60\% of the code elements being equal to zero. These point clouds represent the tradeoff between sparsity and fidelity to the original targets (eg proximity to the global solution as defined in original the convex problem). For each $T$, the (black) LSALSA point-cloud is generally further to the right and/or located below the other point-clouds, representing higher sparsity and/or lower error, respectively. For example, while FISTA achieves some mildly sparser solutions for $T=10, 20$, it significantly sacrifices RMSE. In this sense, we argue that LSALSA enjoys the best sparsity-accuracy trade-off from among the four methods.

Similarly as before, we performed an evaluation on the classification task. A separate classifier was trained for each data set using the separated optimal codes $\xv_1^{*,p}$ and $\xv_2^{*,p}$, respectively. As before, a random Gaussian projection was used to reduce the ASIRRA codes to the length $500$ before inputting to the classifier. The classification results are depicted in Table \ref{tab: cdmnist_mnistTable} for MNIST and Table \ref{tab: cdmnist_cdTable} for ASIRRA. 

Finally, in Figure~\ref{recon: allmini346} we present exemplary reconstructed images obtained by different methods when performing source separation (more reconstruction results can be found in Section~\ref{sec:F} in the Supplement).
FISTA and LISTA are unable to separate components without severely corrupting the ASIRRA component. LSALSA has visually recognizable separations even at $T=1$, and the MNIST component is almost gone by $T=5$. Recall that no additional learning is employed to generate reconstructions, they are simply codes multiplied by corresponding dictionary matrices.

 \begin{table}[ht]\vspace{-0.5cm}
 	\centering
 	\begin{tabular}{ |c ||c|c|c|c|}
 		\hline
 		\multicolumn{1}{|c||}{ } & \multicolumn{4}{|c|}{Classification Error (in \%)}\\\cline{2-5} 
 		Iter & FISTA & LISTA & SALSA  & LSALSA 	\\[1ex]
 		\hline
 		1   & 70.29   & 16.81   & 27.00   & \textbf{2.37}	\\
 		\hline	
 		3   & 34.09   & 16.10   & 24.45   & \textbf{3.69}	\\
 		\hline	
 		5   & 89.97   & 14.78   & 39.74   & \textbf{1.15}\\
 		\hline		
 		10   & 90.00   & 15.00   & 3.03   & \textbf{0.05}\\
 		\hline		
 		20   & 90.00   & 9.74   & 0.85   & \textbf{0.05}	\\
 		\hline	
 		50   & 1.30   & 6.73   & 0.02   & \textbf{0.02}	\\
 		\hline	
 	\end{tabular}
 	\caption{\label{tab: cdmnist_mnistTable} MNIST classification error obtained after source separation (10 classes). The best performer is in bold.}
 \vspace{-0.2in}
 \end{table}
 \begin{table}[ht]\vspace{-0.5cm}
 \centering
 	\begin{tabular}{ |c ||c|c|c|c|}
 		\hline
 		\multicolumn{1}{|c||}{ } & \multicolumn{4}{|c|}{Classification Error (in \%)}\\\cline{2-5} 
 		Iter & FISTA & LISTA & SALSA  & LSALSA 	\\[1ex]
 		\hline
 		1   & 43.70   & 39.80   & 46.70   & \textbf{32.60}	\\
 		\hline	
 		3   & 41.50   & 40.20   & 42.40   & \textbf{35.30}	\\
 		\hline	
 		5   & 49.80   & 38.80   & 43.80   & \textbf{30.00}	\\
 		\hline	
 		10   & 49.90   & 38.60   & 28.90   & \textbf{22.30}	\\
 		\hline	
 		20   & 45.50   & 37.90   & 23.00   & \textbf{19.10}	\\
 		\hline	
 		50   & 28.40   & 36.40   & \textbf{12.40}   & 12.70	\\
 		\hline	
 	\end{tabular}
 	\caption{\label{tab: cdmnist_cdTable} ASIRRA classification error obtained after source separation(2 classes). The best performer is in bold.}
 \vspace{-0.2in}
 \end{table}


\vspace{-0.5cm}
\section{Conclusions}
\label{sec:conclusions}
In this paper we propose a deep encoder architecture LSALSA, obtained from time-unfolding the Split Augmented Lagrangian Shrinkage Algorithm (SALSA). We empirically demonstrate that LSALSA inherits desired properties from SALSA and outperforms baseline methods such as SALSA, FISTA, and LISTA in terms of both the quality of predicted sparse codes, and the running time in both the single and multiple (MCA) dictionary case. In the two-dictionary MCA setting, we furthermore show that LSALSA obtains the separation of image components faster, and with better visual quality than the separation obtained by SALSA. The LSALSA network can tackle the general single and multiple dictionary coding problems without extension, unlike common competitors. 

We also present a theoretical framework to analyze LSALSA. We show that the forward propagation of a signal through the LSALSA network is equivalent to a truncated ADMM algorithm applied to a new, learned cost function that generalizes the original problem.
We show via the optimality conditions for this new cost function that the LSALSA update is related to a ``learned pseudo-Newton'' update down the original loss landscape, whose descent direction is corrected by a learned modification of the Hessian of the original cost function.
Finally, we extend a very recent Stochastic Alternating Optimization analysis framework to show that a \textit{gradient descent step down the learned loss landscape is equivalent with taking a modified gradient descent step along the original loss landscape}. In this framework we provide conditions under which LSALSA's descent direction modification can speed up convergence.
\begin{figure}[!hb]
	\centering
	\begin{tabular}{ccc}
		\centerIm{0.23\textwidth}{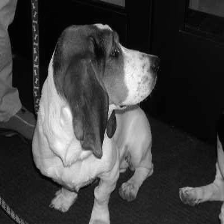}&
		\centerIm{0.23\textwidth}{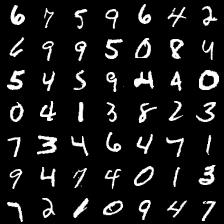}&
		\centerIm{0.23\textwidth}{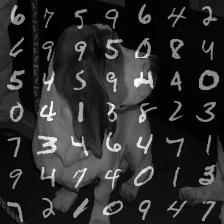}
		\vspace{0.025in}\\
		ASIRRA & MNIST & Mixed Image
	\end{tabular}
	\begin{tabular}{c|c}\hline\\
		\centerIm{0.21\textwidth}{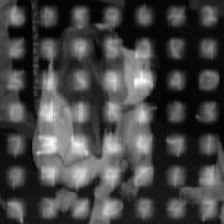}
		\centerIm{0.21\textwidth}{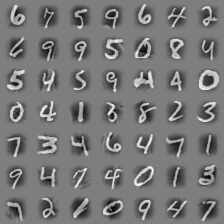}&
		\centerIm{0.21\textwidth}{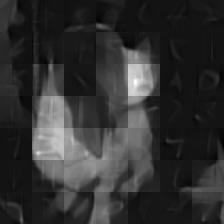}
		\centerIm{0.21\textwidth}{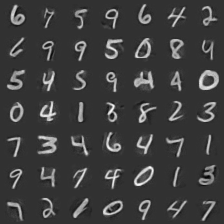}
		\vspace{0.025in}\\
		SALSA $T=1$&LSALSA $T=1$ \\
		\centerIm{0.21\textwidth}{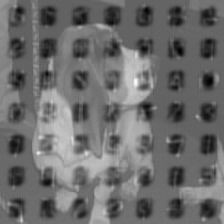}
		\centerIm{0.21\textwidth}{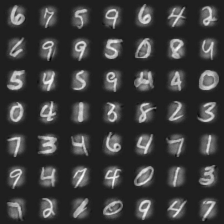}&
		\centerIm{0.21\textwidth}{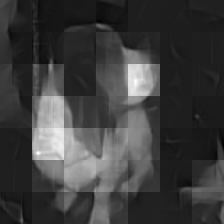}
		\centerIm{0.21\textwidth}{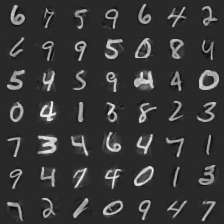}
		\vspace{0.025in}\\
		SALSA $T=5$&LSALSA $T=5$ \\
		\hline
		\\
		\centerIm{0.21\textwidth}{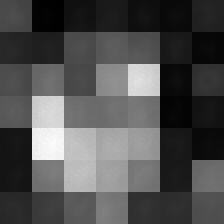}
		\centerIm{0.21\textwidth}{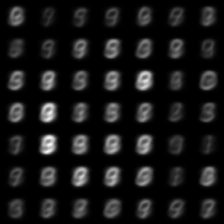}&
		\centerIm{0.21\textwidth}{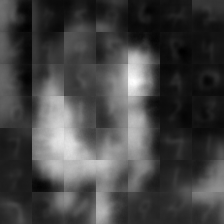}
		\centerIm{0.21\textwidth}{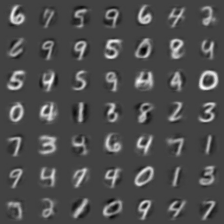}
		\vspace{0.025in}\\
		FISTA $T=1$&LISTA $T=1$ \\
		\centerIm{0.21\textwidth}{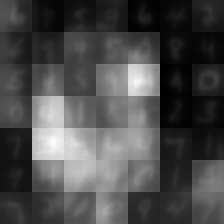}
		\centerIm{0.21\textwidth}{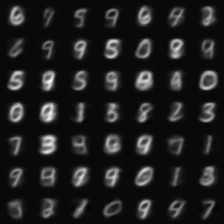}&
		\centerIm{0.21\textwidth}{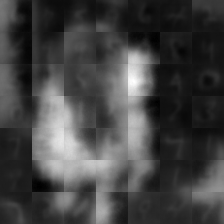}
		\centerIm{0.21\textwidth}{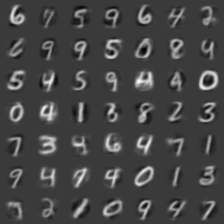}
		\vspace{0.025in}\\
		FISTA $T=5$&LISTA $T=5$ \\ \hline
	\end{tabular}
	\caption{\label{recon:zoombasset} MCA experiment using MNIST + ASIRRA. Image reconstructions obtained by SALSA, LSALSA, FISTA, LISTA for $T = 1,5$. Top row: original data (components and mixed).}
\end{figure}


\bibliographystyle{spmpsci}      

\bibliography{ms-bib}


\clearpage
\newpage

\normalsize
\onecolumn

\appendix
\vbox{
	\hsize\textwidth\linewidth\hsize
	\centering{
		\Large\bf LSALSA: Accelerated Source Separation via Learned Sparse Coding\\ (Supplementary material) \par
	}
}
\section{Sum of $\ell_1$-norms}\label{sec:suppL1}
The proximity operator of $f_2$ (the weighted sum of $\ell_1$-norms as given in Equation~\ref{eqn: mca}) is separable with respect to signal components $\xv_i$:
\begin{align}
    \text{prox}_{f_2}(\zv) :&= \argmin{\xv} \left(\sum_i^D \alpha_i\norm{\xv_i}_1\right)  + \tfrac12\norm{\zv-\xv}_2^2\\
                          &= \argmin{\xv} \sum_i^D \left(\alpha_i\norm{\xv_i}_1  + \tfrac12\norm{\zv_i-\xv_i}_2^2\right),
\end{align}
thus simplifying to $D$ element-wise soft thresholding operators:
\begin{align}
    [\text{prox}_{f_2}(\zv)]_i &= \argmin{\xv_i} \alpha_i\norm{\xv_i}_1  + \tfrac12\norm{\zv_i-\xv_i}_2^2\\
                               &= \soft(\zv_i; \alpha_i),
\end{align}
for $i=1,...,D$, where vector-valued soft-thresholding is defined elementwise in Equation~\ref{eqn: shrink}.

\section{Recursion Equation Derivation}
\label{sec:A_recursion}
We seek a formula for nonlinearity output $u(t)$ in terms of filtered data $\Wve$ and previous layer outputs $u(j), j<t.$ In this section only, we use the simplified notation $u_t\coloneqq u(t)$ for $u$ at the $t-th$ layer.

\begin{align*}
    x_t + d_t &= x_t+(d_{t-1}-u_t+x_t)\\
              &= 2\Sv\left[\Wve\yv+\mu(u_t-d_{t-1})\right]+d_{t-1}-u_t\\
              &=2\Sv\Wve\yv + (2\mu\Sv-I)u_t + (I-2\mu\Sv)d_{t-1}\\
              &=2\Sv\Wve\yv + (2\mu\Sv-I)u_t + (I-2\mu\Sv)\left[d_{t-2}-u_{t-1}+x_{t-1}\right]\\
              &=2\Sv\Wve\yv + (2\mu\Sv-I)u_t + (I-2\mu\Sv)\left[d_{t-2}-u_{t-1}\right]\\
              &\hspace{3.4cm}+(I-2\mu\Sv)\Sv\left[\Wve\yv+\mu(u_{t-1}-d_{t-2})\right],
\end{align*}
where we have expanded $x_{t-1}$ in the last line. Collecting like terms:
\begin{align*}
    x_t + d_t = &(2I+[I-2\mu\Sv])\Sv\Wve\yv \\
    &+ (2\mu\Sv-I)u_t + (I-2\mu\Sv)(\mu\Sv-I)u_{t-1}\\
    &+(I-2\mu\Sv)(I-\mu\Sv)d_{t-2}.
\end{align*}
Expanding just the last term and re-collecting terms once again:
\begin{align*}
    x_t+d_t = &(2I+[I-2\mu\Sv]+[I-2\mu\Sv][I-\mu\Sv])\Sv\Wve\yv\\
              &+ (I-2\mu\Sv)\left[ -u_t + (\mu\Sv-I)u_{t-1} + (I-\mu\Sv)(\mu\Sv-I)u_{t-2}\right]\\
              &+(I-2\mu\Sv)(I-\mu\Sv)^2
                    \underbrace{\left[
                                (I-\mu\Sv)d_{t-4}+(\mu\Sv-I)u_{t-3}+\Sv\Wve\yv
                                \right]}_{d_{t-3}}\\
            = &(2I+[I-2\mu\Sv]+[I-2\mu\Sv][I-\mu\Sv]+[I-2\mu\Sv][I-\mu\Sv]^2)\Sv\Wve\yv\\
            &+(I-2\mu\Sv)[ -u_t + (\mu\Sv-I)u_{t-1}\\
            &\hspace{1cm}+ (I-\mu\Sv)(\mu\Sv-I)u_{t-2}+ (I-\mu\Sv)^2(\mu\Sv-I)u_{t-3}]\\
            &+(I-2\mu\Sv)(I-\mu\Sv)^3d_{t-4}.
\end{align*}
The pattern has emerged. Let $M=[I-\mu\Sv]$, and after expanding the $d_j,x_j$ as done above $p-1$ times, i.e. for $j=t,...,(t-p+1)$ we have:
\begin{align*} 
    x_t+d_t = &\left(2I+[I-2\mu\Sv]\sum_{n=0}^{p-2}M^n\right)\Sv\Wve\yv\\
              &+(I-2\mu\Sv)\left[ -u_t + Mu_{t-1} -\sum_{n=2}^{p-1} M^n u_{t-n}\right]\\
              &+(I-2\mu\Sv)(I-\mu\Sv)^{p-1} d_{t-p}.
\end{align*}
To complete the expansion, let $p=t-1$:
\begin{align*} 
    x_t+d_t = &\left(2I+[I-2\mu\Sv]\sum_{n=0}^{t-3}M^n\right)\Sv\Wve\yv\\
              &+(I-2\mu\Sv)\left[ -u_t + Mu_{t-1} -\sum_{n=2}^{t-2} M^n u_{t-n}\right]\\
              &+(I-2\mu\Sv)(I-\mu\Sv)^{t-2}
                                \underbrace{\left[
                                    -Mu_1+\Sv\Wve\yv
                                \right]}_{d_1}
\end{align*}
where $d_0=0$. Collecting like terms we can absorb the final term; then using the definition of $u_t:$
\begin{align*}
    u_t &= \soft\{x_t+d_t\}\\
        &= \soft\left\{
                   \left(2I+[I-2\mu\Sv]\sum_{n=0}^{t-2}M^n\right)\Sv\Wve\yv
                    +(I-2\mu\Sv)\left[ -u_t + Mu_{t-1} -\sum_{n=2}^{t-1} M^n u_{t-n}\right]
                \right\}
\end{align*}

Clearly, in the case of SALSA, the non-linearity output $\uv(t+1)$ has a complex dependence on \textit{all} of the previous layers' outputs. This dependence comes from the auxiliary variable $\dv$, i.e. the Lagrangian multipliers term. In contrast, the recursion formula for ISTA follows directly from its two-step definition~\cite{cvxOpt:daub04}:
\begin{align}
    \uv(t+1)=\soft(\tfrac1L\Av\Tt\yv+\Slista\uv(t)),
\end{align}
where $\tilde{\Sv}:=I-\tfrac1L A\Tt A$.
Compared to the corresponding equation of SALSA given above, the matrix $\tilde{\Sv}$ plays a limited role in the ISTA/LISTA update. The difference between LSALSA and LISTA is a direct consequence of a different nature of their maternal algorithms, SALSA and ISTA respectively. ISTA is a proximal gradient method that solves the optimization problem of Equation~\ref{lasso} by iteratively applying gradient descent step followed by soft thresholding. SALSA on the other hand is a second-order method that recasts the problem in terms of constrained optimization and optimizes the corresponding Augmented Lagrangian. Consequently, LISTA has a simple structure such that each layer depends only on the previous layer and re-injection of the filtered data $\Wlista\yv$. LSALSA has cross-layer connections resulting from the existence of the Lagrangian multiplier update (the $\dv$-step) in the SALSA algorithm, which allows for learning dependencies between non-adjacent layers.

\section{Derivation of Reparameterization}\label{supp:derivReparam}
We suppose that $\xv$ minimizes \textit{some} Augmented Lagrangian that has the same regularizer and constraint as our original problem, but with a new data-fidelity term $\hat{f_1}(\xv)$:
\begin{align}
\label{optCond1}
    \xv\iter{k+1} &= \argmin{x}\ \ell^A(x,u\iter{k+1};d\iter{k}),\\
    \ell^A(\xv,\uv;\dv) &= \hat{f_1}(\xv) + f_2(\uv) + \frac\mu2\|\uv-\xv-\dv\|_2^2\\
\label{nabPhiStruct}
    \Rightarrow \nabla_{\xv}&\ell^A(\xv,\uv;\dv) = \nabla_{\xv}\hat{f_1}(\xv) + \mu(\xv+\dv-\uv)
\end{align}
The optimality condition for the minimizing $\ell^A$ w.r.t. $\xv$ is:
\begin{align}
\label{optCond2}
    0 &= \nabla \ell(\xv\iter{k+1}, \uv\iter{k+1}; \dv\iter{k}).
\end{align}
And recall the formula for $\xv\iter{k+1}$:
\begin{align}
    \xv\iter{k+1} = &\Sv\left(
                        \Wve\yv+\mu(\uv\iter{k+1}-\dv\iter{k})
                       \right)\\
\label{defxL}
    \Rightarrow 0 =  &\Sv^{-1} \xv\iter{k+1} - \Wve\yv -\mu\left[
                                                                \uv\iter{k+1}-\dv\iter{k}
                                                            \right],
\end{align}
Where we multiplied both sides by $\Sv^{-1}$ and subtracted the right-hand-side from both sides of the equation. Combining Equations \ref{defxL} and \ref{optCond2} yields:
\begin{equation}
\label{GradientDefn}
    \nabla \ell(\xv\iter{k+1}, u\iter{k+1}; d\iter{k})
    =  \Sv^{-1} \xv\iter{k+1} - \Wve y -\mu\left(u\iter{k+1}-d\iter{k}\right),
\end{equation}
And then we plug in the imposed structure from Equation \ref{nabPhiStruct} to discover a formula for $\hat{f_1}$:
\begin{align}
     \nabla_x\hat{f_1}(\xv\iter{k+1}) + \mu(\xv\iter{k+1}+d\iter{k+1}-u\iter{k})
      &=  \Sv^{-1} \xv\iter{k+1} - \Wve \yv +\mu\left(u\iter{k+1}-d\iter{k}\right) \nonumber\\
\Rightarrow \nabla_x\hat{f_1}(x) &= 
      [\Sv^{-1}-\mu I] \xv\iter{k+1}-\Wve \yv
\end{align}
And integrating both sides gives us:
\begin{align}
    \hat{f_1}(\xv) = \xv\Tt[\Sv^{-1}-\mu I]\xv -(\Wve \yv)\Tt \xv + c_0,
\end{align}
where $c_0$ is an arbitrary constant. Setting $c_0=\yv\Tt\yv$ gives us the generalization we seek.

\begin{proof}[Theorem \ref{Thm:reparam}]
Note that $\hLc$ and $\Lc_A$ have identical dependence on $\uv$ and $\dv$, and thus LSALSA and SALSA share optimal update formulas in these directions. This is illustrated in Equation~\ref{eqn: LSrelate} (i.e., $\hLc$ only alters the loss landscape along the $\xv$ directions, using $\Sv$). Thus, we only need to show optimality for the new $\xv$-update. The optimality condition for $\xv\iter{k+1}$ with respect to $\hLc$ is:
\begin{align*}
    0 &= \nabla_{\xv} \hLc(\xv\iter{k+1},\uv\iter{k+1};\dv\iter{k})\\
      &= \left[\Sv^{-1}-\mu I\right]\xv\iter{k+1}-\Wve \yv + \mu(\xv\iter{k+1}-\uv\iter{k+1}+\dv\iter{k})\\
      &= \Sv^{-1}\xv\iter{k+1}-\left[\Wve \yv + \mu(\uv\iter{k+1}-\dv\iter{k})\right],
\end{align*}
then solving for $\xv\iter{k+1}$ yields the $\xv$-update in Algorithm \ref{alg: fprop2}, completing the proof.
\end{proof}

\section{Convergence Statement}\label{supp:convAnalyz}
We use the following notation to abbreviate evaluation of the function $\Lu{}$ at the $d^{th}$ step of the algorithm, within the $t^{th}$ iteration:
\vspace{-0.25cm}
\begin{equation}
    \Lu{x_d}^t= \Lu{}(x_1^{t+1},...,x_{d-1}^{t+1}, x_d^t, x_{d+1}^t,...,x_K^t) ,
\end{equation}
i.e. where $\{x_i\}_{i=1}^{d-1}$ have been updated and $\{x_i\}_{i=d}^{K}$ have not. Finally let us denote the cost function evaluated at optimal points everywhere except the $d^{th}$ argument with:
\begin{equation}
    \Lu{x_d}^* = \Lu{}(x_1^*, x_2^*, ..., x_{d-1}^*, x_d, x_{d+1}^*, ..., x_{K}^*).
\end{equation}
In general, if no superscript is present, $\Lu{x_d}$ indicates that we are holding all $x_{i\neq d}$ fixed and only considering a function of the $d^{th}$ argument only.

\vspace{-0.25cm}
\subsection{Statement of Theorems and Results}
\vspace{-0.25cm}
Given the Assumptions listed in~\ref{sec:assumps}, the following theorems can be derived from the results in~\cite{sao:beyondBackprop}. Though our Theorem~\ref{thm:recursion} closely follows the steps of Theorem 3.1 from~\cite{sao:beyondBackprop}, there are some minor differences where we allow ascent steps in addition to descent steps, so a proof is provided in Supplementary Section~\ref{supp:convAnalyz}. Our Theorem~\ref{thm:mainResult} follows directly from Theorem 3.2 in~\cite{sao:beyondBackprop}, so we do not provide a proof.

\begin{theorem}\label{thm:recursion}
    Given the Stochastic Alternating Optimization scheme in Equation~\ref{eqn:propIt} for solving the saddle point problem of Equation~\ref{eq:sadp} and a decaying step size $\{\eta^t\}_{t = 0}^{\infty}$, the error at the $(t+1)^{th}$ iteration, denoted by $\Delta_d^{t+1}\coloneqq x_d^{t+1}-x_d^*$, satisfies the following recursion:
    \begin{equation}
        \sum_{d=1}^K\norm{\Delta_d^{t+1}}^2 \leq (1-q^t)\sum_{d=1}^K\norm{\Delta_d^t}^2+\frac{(\eta^t\sigma)^2}{1-\eta^t\gamma_d(K-1)},
    \end{equation}
    where $q^t=1-\frac{1-2\eta^t\xi+\eta^t\gamma(2K-1)}{1-\eta^t\gamma K}\in(0,1)$ for $\gamma<\frac{2\xi}{3K-1}$.
\end{theorem}

\begin{theorem}\label{thm:mainResult}
     Given the Stochastic Alternating Optimization scheme in Equation~\ref{eqn:propIt} for solving the saddle point problem of Equation~\ref{eq:sadp} and a decaying step size\\ $\eta^t = \frac{3/2}{[2\xi-\gamma(2K-1)](t+2)+\frac32\gamma K}$ and assuming $\gamma<\frac{2\xi}{3K-1}$, the error at iteration $t+1$ satisfies
    \begin{equation}\label{eqn:officialErrBound}
    \E{\sum_{d=1}^K\norm{\Delta_d^{t+1}}^2}\leq\left(\frac{2}{t+3}\right)^{\frac32}\E{\sum_{d=1}^K\norm{\Delta_d^0}^2} + \frac{9\sigma^2}{[2\xi-\gamma(2K-1)]^2(t+3)},
    \end{equation}
    where $\xi\coloneqq \min_d\frac{2\alpha_d\beta_d}{\alpha_d+\beta_d}$ depends on the convexity, concavity, and smoothness modulii.
\end{theorem}

First we mention a classical result~\cite{nesterov2013introductory}, necessary for the following proofs.
\begin{lemma}[Bounded Gradient Near Solution]
Combining all Assumptions from Section~\ref{sec:assumps}, we have for each $d$ that
\begin{equation}
    \norm{G_{x_d}^{*,\pm}-x_d^*}\leq\left(1-\frac{2\eta\alpha_d\beta_d}{\alpha_d+\beta_d}\right)\norm{x_d-x_d^*},
\end{equation}
where
\begin{equation}
    G_{x_d}^{t,\pm} = x_d \pm \eta^t\nabla_d\Lu{x_d}^t
\end{equation}
is the gradient (ascent or descent) operator.
\end{lemma}

The following is an extension of Theorem A.1 of~\cite{sao:beyondBackprop}. It holds for each $d=1,...,K$ with $r_d>0$ and $x_d\in B_2(r_d,x_d^*).$
\begin{corollary}[Extension of Theorem A.1~\cite{sao:beyondBackprop}]\label{corr:annaThmA1}
For some radius $r_d>0$ and triplet $(\gamma_d, \beta_d, \alpha_d)$ such that $0\leq \gamma_d \leq \beta_d \leq \alpha_d$, suppose Assumptions A.1-A.5 hold ($\beta_d$-convexity for $d\leq K_1$, $\beta_d$-concavity for $d>K_1$, $\alpha_d$-smoothness, and $GS(\gamma_d)$). Then, for a stepsize $0<\eta\leq\min_d\frac{2}{\alpha_d+\beta_d}$ over $d\in\{1,...K\}$, the population gradient operator $G_d^{\pm}$ is contractive over the ball $B_2(\frac{r_d}{2},x_d^*),$ for the function $\Lu{}$. That is,
\begin{align}
    \norm{G_d^{\pm}-x_d^*}\leq (1-\eta\xi)\norm{x_d-x_d^*}+\eta\gamma\sum_{i\neq d} \norm{x_i-x_i^*},
\end{align}
where $\xi\coloneqq \min_d\frac{2\alpha_d\beta_d}{\alpha_d+\beta_d}$, and $\gamma\coloneqq\max_d \gamma_d$ over $d\in\{1,...,K\}$.
\end{corollary}

\begin{proof}[Corollary \ref{corr:annaThmA1}]
    \begin{align}
    \norm{G_d^{\pm}-x_d^*}
    &= \norm{x_d\pm\eta\nabla_d\Lu{x_d}-x_d^*}\\
    &= \norm{x_d\pm\eta\nabla_d\Lu{x_d}-x_d^* +\eta\nabla_d\Lu{x_d}^*-\eta\nabla_d\Lu{x_d}^*}\\
\label{eqn:triInEq0}
    &\leq \norm{x_d\pm\eta\nabla_d\Lu{x_d}^*-x_d^*}
     +\eta\norm{\pm\nabla_d\Lu{x_d}\mp\nabla_d\Lu{x_d}^*}\\
    &= \norm{G_d^{*,\pm} - x_d^*} + \eta\norm{\nabla_d\Lu{x_d}-\nabla_{d}\Lu{x_d}^*},
\end{align}
where for (\ref{eqn:triInEq0}) we used the triangle inequality, and use the $\pm/\mp$ symbols to indicate that in the $d\leq K_1$ case of gradient descent ($G_d^-$), we bring the $-\eta\nabla_d\Lu{x_d}^*$ term to the left norm; in the $d>K_1$ case we bring the positive version to the left norm. Using Lemma A.1 for the left term and Assumption~\ref{asspn:GS}, $GS(\gamma_d)$, for the right term, we have:
\begin{align}
    \norm{G_d^{\pm}-x_d^*}\leq (1-\eta\xi)\norm{x_d-x_d^*}+\eta\gamma\sum_{i\neq d} \norm{x_i-x_i^*},
\end{align}
where $\xi\coloneqq \min_d\frac{2\alpha_d\beta_d}{\alpha_d+\beta_d}$, and $\gamma\coloneqq\max_d \gamma_d$ over $d\in\{1,...,K\}$. 
\end{proof}
\begin{proof}[Theorem~\ref{thm:recursion}]
First, we use the fact that the error before the projection to the ball around the solution $B_2(r_d,\xv^*_d)$ is greater than or equal to the error after projection. Let $x_d=\Pi_d(\tilde{x}_d)$:
\begin{align*}
\norm{\Delta_d^{t+1}}^2-\norm{\Delta_d^t}^2
    &\leq\norm{\tilde{\Delta}_d^{t+1}}^2-\norm{\Delta_d^t}^2\\
    &=\norm{\tilde{x}_d^{t+1}-x_d^*}^2-\norm{x_d^t-x_d^*}^2\\
    &= \langle \tilde{x}_d^{t+1}-x_d^t, \tilde{x}_d^{t+1}+x_d^t-2x_d^* \rangle\\
    &=\bigg \langle \left[x_d^t\pm\eta^t\nabla_d^1\Lu{x_d}^t\right] -x_d^t, 
              \left[x_d^t\pm\eta^t\nabla_d^1\Lu{x_d}^t\right]+x_d^t-2x_d^* \bigg\rangle,
\end{align*}
where we use $\pm$ to denote that $\tilde{x}_d^{t+1}$ is determined by a gradient ascent step for $d>K_1$ and by a gradient descent step for $d\leq K_1$; and we use $\nabla^1$ to denote the gradient computed from a single sample. This can be simplified to:
\begin{align*}
    &=\big\langle \pm\eta^t\nabla^1_d \Lu{x_d}^t, \pm\eta^t\nabla^1_d \Lu{x_d}^t + 2(x_d^t-x_d^*) \big \rangle \\
    &= (\eta^t)^2\norm{\nabla^1_d \Lu{x_d}^t}^2 \pm 2\eta^t \langle \nabla^1_d \Lu{x_d}^t, \Delta_d^t \rangle.
\end{align*}

Now we take the expected value of both sides after a rearrangement:
\begin{align}
\label{eqn:lastTermRef}
    \E{\norm{\Delta_d^{t+1}}^2} \leq \E{\norm{\Delta_d^t}^2} + (\eta^t)^2\E{\norm{\hat{W}_d^t}^2} \pm 2\eta^t\E{\langle \hat{W}_d^t, \Delta_d^t \rangle},
\end{align}
using the abbreviation $\hat{W}_d^t=\nabla^1_d \Lu{x_d}^t.$

Now, note that for $d\leq K_1$, the function is convex and thus $(x^*_d - x^t_d) = -\Delta_d^t$ is a \textit{descent direction}. Conversely, for $d>K_1$, the function is concave and $(x^*_d - x^t_d)=-\Delta_d^t$ is an \textit{ascent direction}. This means that
\begin{align}
\label{eqn:descent}
    \langle W_d^*, &\Delta_d^t\rangle \geq 0, \ \text{if } d\leq K_1\ \text{(gradient descent case)}\\
\label{eqn:ascent}
    \langle W_d^*, &\Delta_d^t\rangle \leq 0, \ \text{if } d > K_1\ \text{(gradient ascent case)},
\end{align}
where $W^*_d=\nabla_d \Lu{x_d^*}^*$. We expand the last term of (\ref{eqn:lastTermRef}) for each case separately:\\
\textbf{For} $\mathbf{d\leq K_1}$:
\begin{align*}
-\eta^t\E{\langle \hat{W}_d^t, \Delta_d^t \rangle} 
    &\leq -\eta^t\E{\langle \hat{W}_d^t, \Delta_d^t \rangle} + \eta^t\E{\langle W_d^*, \Delta_d^t\rangle}\\
    &=-\eta^t\E{\langle \hat{W}_d^t-W_d^*, \Delta_d^t \rangle},
\end{align*}
using (\ref{eqn:descent}). Note that for $d\leq K_1$, we have $G_d^{t,-}=x_d^t-\eta^t\hat{W}_d^t$, implying that $\eta^t\hat{W}_d^t = -(G_d^{t,-}-x_d^t)$. We use this to continue:
\begin{align}
    -\eta^t\E{\langle \hat{W}_d^t, \Delta_d^t \rangle}  
    &\leq -\eta^t\E{\langle \hat{W}_d^t-W_d^*, \Delta_d^t \rangle} \nonumber\\
    &= -\E{\langle -[G_d^{t,-}-x_d^t]+[G_d^{*,-}-x_d^*], \Delta_d^t \rangle} \nonumber\\
    &= \E{\langle [G_d^{t,-}-x_d^t]-[G_d^{*,-}-x_d^*], \Delta_d^t \rangle} \nonumber\\
    &= \E{\langle [G_d^{t,-}-G_d^{*,-}]-[x_d^t-x_d^*], \Delta_d^t \rangle} \nonumber\\
\label{descentCase}
    &= \E{\langle [G_d^{t,-}-G_d^{*,-}], \Delta_d^t \rangle} - \E{\norm{\Delta_d^t}^2}
\end{align}
\textbf{For} $\mathbf{d> K_1}$, we similarly have:
\begin{align*}
+\eta^t\E{\langle \hat{W}_d^t, \Delta_d^t \rangle} 
    &\leq +\eta^t\E{\langle \hat{W}_d^t, \Delta_d^t \rangle} - \eta^t\E{\langle W_d^*, \Delta_d^t\rangle}\\
    &=+\eta^t\E{\langle \hat{W}_d^t-W_d^*, \Delta_d^t \rangle},
\end{align*}
using (\ref{eqn:ascent}). For $d>K_1$ we have $G_d^{t,+}=x_d^t+\eta^t\hat{W}_d^t$, implying that $\eta^t\hat{W}_d^t = G_d^{t,+}-x_d^t$. This leads us to
\begin{align}
    \eta^t\E{\langle \hat{W}_d^t, \Delta_d^t \rangle} 
    &\leq \eta^t\E{\langle \hat{W}_d^t-W_d^*, \Delta_d^t \rangle} \nonumber\\
    &= \E{\langle [G_d^{t,+}-x_d^t]-[G_d^{*,+}-x_d^*], \Delta_d^t \rangle} \nonumber\\
    &= \E{\langle [G_d^{t,+}-G_d^{*,+}]-[x_d^t-x_d^*], \Delta_d^t \rangle} \nonumber\\
\label{ascentCase}
    &= \E{\langle [G_d^{t,+}-G_d^{*,+}], \Delta_d^t \rangle} - \E{\norm{\Delta_d^t}^2}
\end{align}
Combining (\ref{descentCase}) and (\ref{ascentCase}) with where we left off in (\ref{eqn:lastTermRef}) yields:
\begin{align}
    &\E{\norm{\Delta_d^{t+1}}^2} \leq \E{\norm{\Delta_d^t}^2} + (\eta^t)^2\E{\norm{\hat{W}_d^t}^2} \pm 2\eta^t\E{\langle \hat{W}_d^t, \Delta_d^t \rangle}\\
    &\leq \E{\norm{\Delta_d^t}^2} + (\eta^t)^2\E{\norm{\hat{W}_d^t}^2} +2\E{\langle [G_d^{t,\pm}-G_d^{*,+}], \Delta_d^t \rangle}-2\E{\norm{\Delta_d^t}^2}\\
\label{eqn:leftOff2}
    &= -\E{\norm{\Delta_d^t}^2} + (\eta^t)^2\E{\norm{\hat{W}_d^t}^2} +2\E{\langle [G_d^{t,\pm}-G_d^{*,+}], \Delta_d^t \rangle}.
\end{align}
Next, we simplify the last term of (\ref{eqn:leftOff2}) using Cauchy-Schwarz and then Corollary A.1:
\begin{align}
2&\E{\langle [G_d^{t,\pm}-G_d^{*,+}], \Delta_d^t \rangle} 
    \leq 2\norm{G_d^{t,\pm}-x_d^*}\norm{\Delta_d^t}\\
\label{eqn:useLem1}
    &\leq
    2\left[
        (1-\eta^t\xi)\norm{\Delta_d^t}+\eta^t\gamma
        \left(
        \sum_{i< d}\norm{\Delta_i^{t+1}}+\sum_{i> d} \norm{\Delta_i^t}
        \right)
    \right]\norm{\Delta_d^t}\\
    &= 2(1-\eta^t\xi)\norm{\Delta_d^t}^2+
        \eta^t\gamma
        \left(
        \sum_{i< d}2\norm{\Delta_i^{t+1}}\norm{\Delta_d^t}+\sum_{i> d} 2\norm{\Delta_i^t}\norm{\Delta_d^t}
        \right),
\end{align}
Now we apply the fact that $2ab\leq a^2+b^2$ to each component of each sum to get:
\begin{align}
&2\E{\langle [G_d^{t,\pm}-G_d^{*,+}], \Delta_d^t \rangle} \\
    \leq &2(1-\eta^t\xi)\norm{\Delta_d^t}^2+\eta^t\gamma
        \left(
         \sum_{i=1}^{d-1}\left(\norm{\Delta_i^{t+1}}^2+\norm{\Delta_d^t}^2\right)
        +\sum_{i=d+1}^K\left(\norm{\Delta_i^t}^2    +\norm{\Delta_d^t}^2\right)
        \right)\\
    = &2(1-\eta^t\xi)\norm{\Delta_d^t}^2+\eta^t\gamma
        \left( (K-1)\norm{\Delta_d^t}^2
        +\sum_{i=1}^{d-1}\norm{\Delta_i^{t+1}}^2
        +\sum_{i=d+1}^K\norm{\Delta_i^t}^2    
        \right)\\
    = &\left(2-2\eta^t\xi+\eta^t\gamma(K-1)\right)\norm{\Delta_d^t}^2+\eta^t\gamma\left(\sum_{i=1}^{d-1}\norm{\Delta_i^{t+1}}^2
        +\sum_{i=d+1}^K\norm{\Delta_i^t}^2    
        \right).
\end{align}
Combining this result with (\ref{eqn:leftOff2}) gives us:
\begin{equation}
\E{\norm{\Delta_d^{t+1}}^2}
\leq-\E{\norm{\Delta_d^t}^2} + (\eta^t)^2\E{\norm{\hat{W}_d^t}^2} +2\E{\langle [G_d^{t,\pm}-G_d^{*,+}], \Delta_d^t \rangle}
\end{equation}
\begin{equation}
    \begin{aligned}
        &\leq \left(1-2\eta^t\xi+\eta^t\gamma(K-1)\right)\E{\norm{\Delta_d^t}^2} 
                + (\eta^t\sigma_d)^2 \\
        &\hspace{1cm}+\eta^t\gamma\E{
                \left(\sum_{i<d}\norm{\Delta_i^{t+1}}^2
                + \sum_{i>d}\norm{\Delta_i^t}^2    
            \right)},
    \end{aligned}
\end{equation}
where we additionally call on Assumption~\ref{asspn:BG} to statistically bound the gradient with  $$\E{\norm{\hat{W}_d^t}^2}\leq\sigma_d^2.$$

We next expand terms to bound the sum of errors (after moving all $\Delta^{t+1}$ terms to the left-hand-side):
\begin{equation*}
    \begin{aligned}
    \E{\sum_{d=1}^{K}\norm{\Delta_d^{t+1}}^2} - 
    \eta^t\gamma\E{\sum_{d=1}^K\sum_{i<d}\norm{\Delta_i^{t+1}}^2}
    \leq 
    &\left(1-2\eta^t\xi+\eta^t\gamma(K-1)\right)\E{\sum_{d=1}^{K}\norm{\Delta_d^t}^2}\\
    &+\eta^t\gamma\E{\sum_{d=1}^K\sum_{i>d}\norm{\Delta_i^t}^2} + \sum_{d=1}^K (\eta^t\sigma_d)^2.
    \end{aligned}
\end{equation*}
Note that for any subset $S\subseteq\{1,...,K\}$,
\begin{align}
    \sum_{d=1}^K\sum_{i\in S}\norm{\Delta_i^t}^2 &\leq \sum_{d=1}^K\sum_{i=1}^K\norm{\Delta_i^t}^2 = K\sum_{d=1}^K\norm{\Delta_d^t}^2.
\end{align}
And thus, using $S=\{1,...,d-1\}$ on the left-hand-side we get
\begin{equation}
\label{eqn:resultLHS}
    (1-\eta^t\gamma K)\E{\sum_{d=1}^{K}\norm{\Delta_d^{t+1}}^2}
    \leq 
    \E{\sum_{d=1}^{K}\norm{\Delta_d^{t+1}}^2} - 
    \eta^t\gamma\E{\sum_{d=1}^K\sum_{i<d}\norm{\Delta_i^{t+1}}^2},
\end{equation}
and using $S=\{d+1,...,K\}$ on the right-hand-side we get
\begin{equation}
\label{eqn:resultRHS}
\begin{aligned}
    \left(1-2\eta^t\xi+\eta^t\gamma(K-1)\right)&\E{\sum_{d=1}^{K}\norm{\Delta_d^t}^2}
    +\eta^t\gamma\E{\sum_{d=1}^K\sum_{i>d}\norm{\Delta_i^t}^2}\\
    \leq
    \left(1-2\eta^t\xi+\eta^t\gamma(2K-1)\right)&\E{\sum_{d=1}^{K}\norm{\Delta_d^t}^2}
\end{aligned}
\end{equation}
Putting together the results in (\ref{eqn:resultLHS}) and (\ref{eqn:resultRHS}), then dividing by $1-\eta^t\gamma K$ concludes the proof:
\begin{equation}
\begin{aligned}
    \E{\sum_{d=1}^{K}\norm{\Delta_d^{t+1}}^2} &\leq
    \frac{1-2\eta^t\xi+\eta^t\gamma(2K-1)}{1-\eta^t\gamma K}
    \E{\sum_{d=1}^K\norm{\Delta_d^t}^2}+\frac{(\eta^t\sigma)^2}{1-\eta^t\gamma K}\\
    &=
    \left(1-q^t\right)
    \E{\sum_{d=1}^K\norm{\Delta_d^t}^2}+\frac{(\eta^t\sigma)^2}{1-\eta^t\gamma K},
\end{aligned}
\end{equation}
where $q^t=1-\frac{1-2\eta^t\xi+\eta^t\gamma(2K-1)}{1-\eta^t\gamma K}.$
Note that in the last step, we have implicitly assumed that $1-\eta^t\gamma K>1$, implying that $\eta^0<\frac{1}{K\gamma}.$ Finally, for convergence we require 
\begin{align}
    \frac{1-2\eta^t\xi+\eta^t\gamma(2K-1)}{1-\eta^t\gamma K}&<1\\
\Rightarrow
    \frac{2\xi}{3K-1}&>\gamma.
\end{align}
However, we also need to require that this factor is greater than zero (not addressed in previous work), which is equivalent to
\begin{align}
    1-2\eta^t\xi+\eta^t\gamma(2K-1)&>0\\
\Rightarrow
    \frac{1}{2\xi+\gamma}&>\eta^t.
\end{align}
So we need $\eta^0<\min\{\frac{1}{K\gamma},\frac{1}{2\xi+\gamma}\}$.
\end{proof}


\begin{proof}[Theorem \ref{thm:chgCvxMod}]
The bound in Inequality (\ref{eqn:officialErrBound}) depends on convexity modulus $\beta$ through the factor $\xi$:
\begin{align}
\frac{9\sigma^2}{[2\xi-\gamma(2K-1)]^2(t+3)}
    &= \mathcal{O}\left(\frac{1}{\xi^2}\right)  \\
    &= \mathcal{O}\left(\frac{\alpha^2+\beta^2}{(2\alpha\beta)^2} +
            \frac{1}{2\alpha\beta}\right).
\end{align}
The derivative $\partial_{\beta}(\xi^2)$ is uniformly positive, so we know that the error bound will decrease in response to an increase $\beta$. We ignore the last $\mathcal{O}(1/2\alpha\beta)$ term for simplicity. Then defining $p(\alpha,\beta):=\frac{\alpha^2+\beta^2}{(2\alpha\beta)^2}$ we complete the proof by computing $p(\alpha,\beta)-p(\alpha,\hat{\beta})$.
\end{proof}

\section{Hyper-parameter search}
\label{sec:AB}
For encoder training, we have done an exhaustive hyperparameter search for the following parameters and ranges:
\begin{itemize}
	\item single dictionary case: $\alpha = \{1eN,2.5eN,5eN,7.5eN\}$ for $N={-1,0,1}$, $\mu = \{.1,1,5,10\}$, batch size = $\{1e2,5e2,1e3\}$, learning rate = $\{1e-7, …,1e-2\}$, learning rate decay = $\{.8,.99,1\}$
	\item multiple dictionary case: $\alpha_1$ and $\alpha_2 = \{1eN, 2.5eN, 5eN,7.5eN\}$ for $N = \{-2,-1,0,1\}$, $\mu = \{0.1,1,10\}$, batch size = $\{100,500\}$, and learning rate = $\{1e-9, …,1e-3\}$. This process was repeated for each method and for each fixed number of iterations ($T$). 
\end{itemize}

\section{Dictionary Learning Experiments}
\label{sec:B}
We visualize the learned dictionary atoms for both single (Figure~\ref{fig:s} and~\ref{fig:ss}) and two dictionary (MCA) cases (Figure~\ref{fig:m}, and~\ref{fig:mm}). 

\subsection{Dictionaries used in single dictionary experiments}\label{supp: dictAtoms}

\begin{figure}[H]
	\centering
	\subfigure[]{\includegraphics[width = 0.32\textwidth]{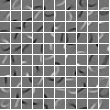}}
	\hfill
	\subfigure[]{\includegraphics[width = 0.32\textwidth]{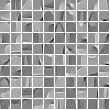}}
	\hfill
	\subfigure[]{\includegraphics[width = 0.32\textwidth]{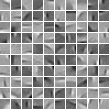}}
	\caption{Visualization of dictionary atoms trained on (a) 10$\times$10 MNIST image patches (b) 10$\times$10 Fashion-MNIST image patches (c) 10$\times$10 CIFAR-10 image patches. Each dictionary has $100$ atoms (complete) and each atom is a unit norm vector of length $100$ reshaped to $10\times10$.}
	\label{fig:s}
\end{figure}
\begin{figure}[htp!]
	\centering
	\includegraphics[width = 0.4\textwidth]{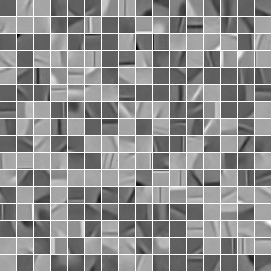}
	\caption{Visualization of ASIRRA dictionary trained on 16$\times$16 image patches. The dictionary has $256$ (complete) atoms and each atom is a unit norm vector of length $256$ reshaped to $16\times16$.}
	\label{fig:ss}
\end{figure}

\subsection{Dictionaries used in MCA experiments}
\label{supp: mca}
In the first set of MCA experiments we performed source separation on $32\times32$ MNIST + ASIRRA images. We used two dictionaries trained independently using whole $32\times32$ MNIST images and $32\times32$ patches of ASIRRA images, after resizing ASIRRA to $224\times224$ (the ASIRRA images come in varying sizes). In the second set of MCA experiments, we performed source separation on spatially added MNIST and CIFAR-10 images (more results of this experiment showed in Section~\ref{sup:AddiMCA_res} of the Supplement). We used same MNIST dictionary as used in MNIST + ASSIRA experiments and trained a CIFAR-10 dictionary on the whole $32\times32$ grayscale CIFAR-10 data set images.These dictionaries have 1024 atoms (complete), all normalized vectors of length 1024 and reshaped to $32\times32$ for visualization. A subset of atoms of the dictionaries used in MCA experiments are visualized in Figure \ref{fig:m} and Figure \ref{fig:mm}. 

\begin{figure}[H]
	\centering
	\includegraphics[width = 0.6\textwidth]{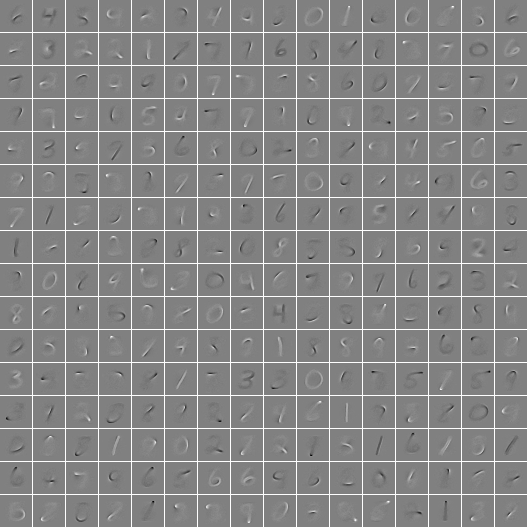}
	\caption{ Visualization of dictionary atoms trained on $32\times32$ MNIST images.}
	\label{fig:m}
\end{figure}

\newpage

\begin{figure}[H]
	\centering
	a) \includegraphics[height = 0.45\textheight]{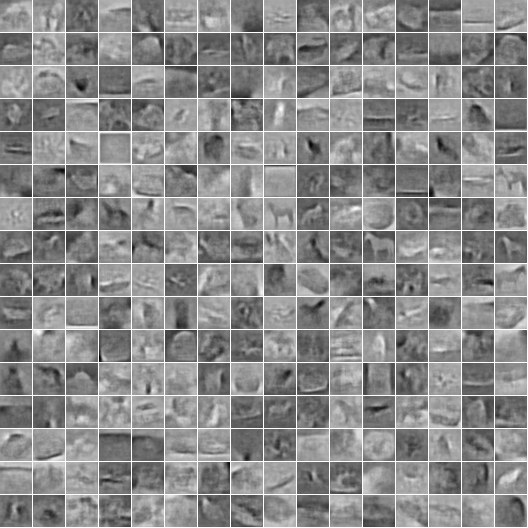}\\
	b) \includegraphics[width = 0.45\textheight]{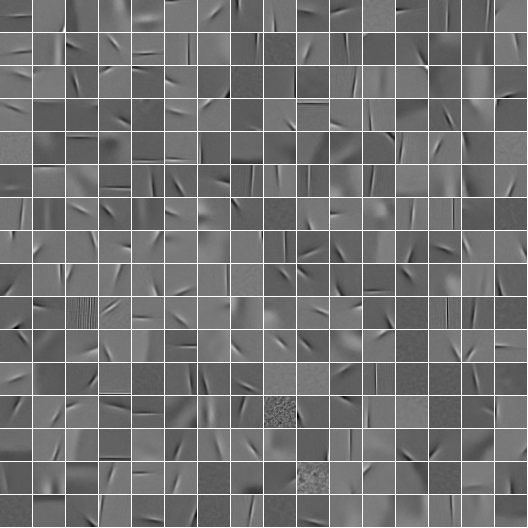}
	\caption{Visualization of dictionary atoms trained on (a)  $32\times32$ CIFAR10 images and (b)  $32\times32$ ASIRRA image patches.}
	\label{fig:mm}
\end{figure}

\FloatBarrier
\section{Additional single dictionary SC experiments}\label{supp: more1Drez}
\label{sec:C}

\subsection{Fashion-MNIST and ASIRRA}\label{supp:fmnistCD}
Here we show classification results from the experiments on Fashion-MNIST and ASIRRA as described in Section~\ref{subsec:sdc}. The classification results are shown in Table~\ref{tab: Fmnist1d} and Figure~\ref{fig: Fmnist1d} for Fashion MNIST, and Table~\ref{tab: catdog1d} and Figure~\ref{fig: catdog1d} for ASIRRA.

 \begin{table}[!h]
 	\centering 
     \begin{tabular}{ |c ||c|c|c|c|}
 		\hline
 		\multicolumn{1}{|c||}{ } & \multicolumn{4}{|c|}{Classification Error (in \%)}\\\cline{2-5} 
 		Iter & FISTA & LISTA & SALSA  & LSALSA\\ [0.5ex]
 		\hline
 		1   & 87.53   & 54.61   & 56.48   & \textbf{11.23}\\
 		\hline
 		5   & 78.46   & 38.13   & 23.61   & \textbf{3.18}	\\
 		\hline
 		7   & 70.25   & 37.16   & 9.20   & \textbf{0.66}	\\
 		\hline
 		10   & 56.06   & 32.90   & 1.59   & \textbf{0.08}	\\
 		\hline
 		15   & 32.00   & 30.45   & \textbf{0.00}   & \textbf{0.00}\\
 		\hline	
 		50   & 0.10   & 14.03   & \textbf{0.00}   & \textbf{0.00}	\\
 		\hline
 		100   & \textbf{0.00}   & 7.95   & \textbf{0.00}   & \textbf{0.00}	\\
 		\hline
 	\end{tabular}
 	\caption{\label{tab: Fmnist1d} Fashion-MNIST classification results. The best performer is in bold. All methods but LISTA were able to match the optimal codes well enough to get zero percent error by $T=100$.}
 \end{table}

\begin{figure}[!h]
    \centering
    \includegraphics[width=0.5\textwidth]{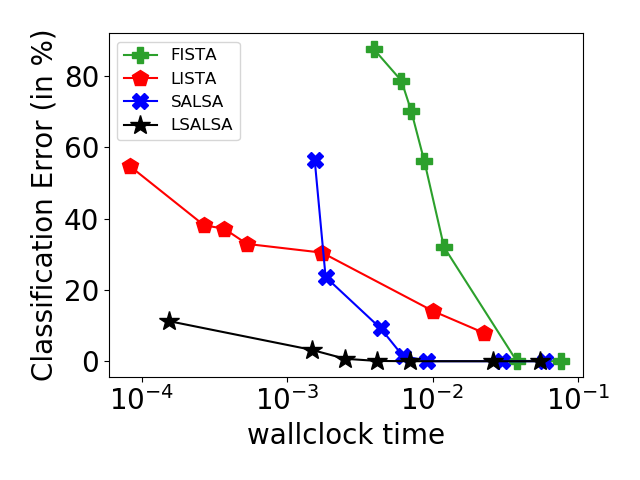}
    \caption{Fashion-MNIST: Classification errors Vs sparse code inference time plotted for different network lengths. All methods but LISTA were able to match the optimal codes well enough to get zero percent error by $T=100$ and LSALSA is from outperforming everyone from beginning.}
    \label{fig: Fmnist1d} 
\end{figure}
	
 \begin{table}[!h] 
 	\centering
 	\begin{tabular}{ |c ||c|c|c|c|}
 		\hline
 		\multicolumn{1}{|c||}{ } & \multicolumn{4}{|c|}{Classification Error (in \%)}\\\cline{2-5} 
 		Iter & FISTA & LISTA & SALSA  & LSALSA \\
 		\hline
 		1   & 48.90   & 52.40   & 48.80   & \textbf{40.10} \\
 		\hline
 		3   & 49.20   & 52.70   & 46.00   & \textbf{42.80} \\
 		\hline
 		5   & 48.50   & 53.50   & 44.80   & \textbf{35.00}  \\
 		\hline
 		7   & 47.80   & 53.70   & 44.50   & \textbf{35.10}	\\
 		\hline
 		10   & 46.50   & 38.50   & 42.70   & \textbf{34.40}	\\
 		\hline
 		15   & 43.90   & 38.10   & 40.70   & \textbf{33.10}	\\
 		\hline
 		20   & 42.10   & 37.60   & 38.70   & \textbf{31.60}	\\
 		\hline
 		50   & 37.80   & 38.20   & 37.20   & \textbf{31.90}	\\
 		\hline
 		100   & 36.40   & 37.10   & 36.80   & \textbf{30.80}	\\
 		\hline
 	\end{tabular}
 	\caption{\label{tab: catdog1d} ASIRRA classification results (2 classes). The best performer is in bold.}
 \end{table}
\begin{figure}[!h]
    \centering
    \includegraphics[width=0.6\textwidth]{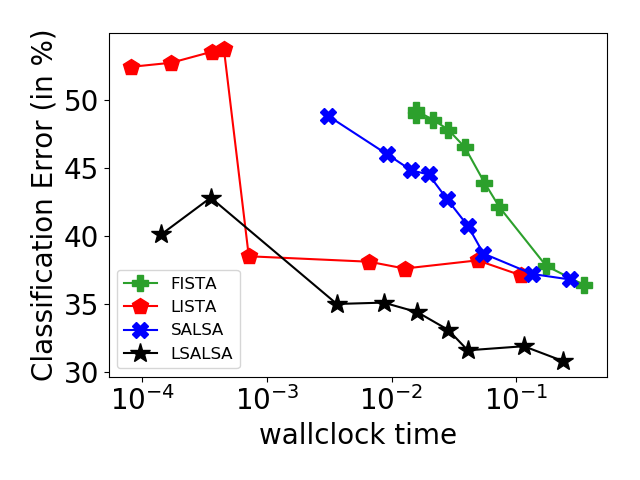}
    \caption{ASIRRA: Classification errors Vs sparse code inference wallclock time plotted for different network lengths.}
    \label{fig: catdog1d} 
\end{figure}

\FloatBarrier
\subsection{MNIST}\label{supp:mnist}
The $32\times32$ MNIST images were first scaled to pixel values in range $[0,1]$ and then divided into $10\times10$ non-overlapping patches (ignoring extra pixels on edges), resulting in $9$ patches per image. Only patches with standard deviation $\geq$ 0.1 were used in training and the remaining ones were discarded (as they are practically all-black). Optimal codes were computed for each vectorized patch by minimizing the objective from Equation~\ref{lasso} by running FISTA for $200$ iterations giving approximately $95\%$ sparse optimal codes, using L1 parameter $\alpha^*=0.15.$
\vspace{-0.2in}
\begin{figure}[H]
	\centering 
	\includegraphics[width = 0.49\linewidth]{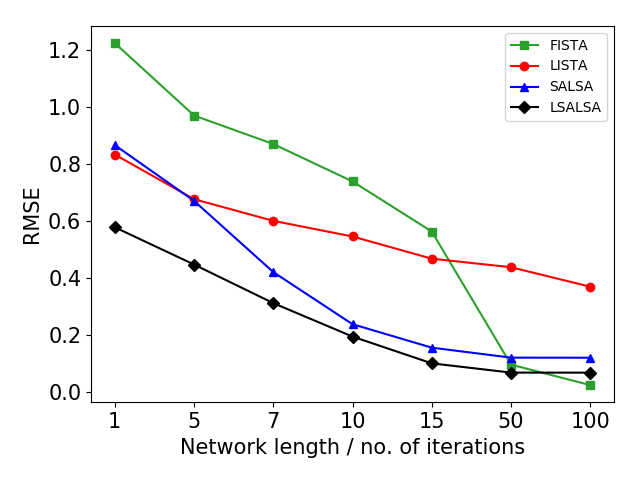}
	\hfill
	\includegraphics[width = 0.49\linewidth]{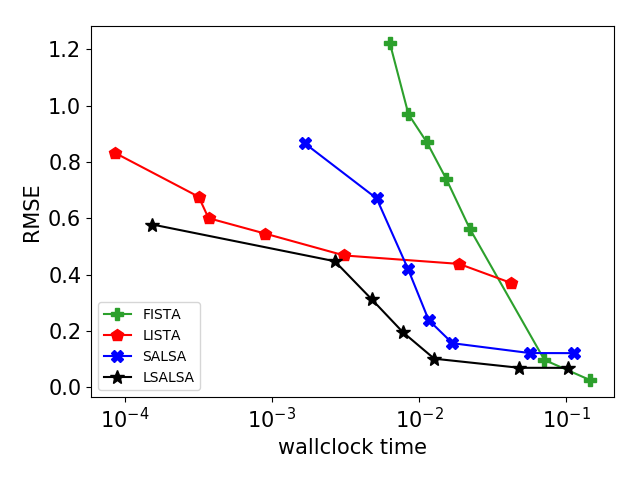}
	\hfill
	\caption{\label{supp:MNIST_1D_res}(\textbf{Left}:) MNIST code prediction errors for varying numbers of iterations. FISTA takes $15$ iterations to give error produced by $1$ iteration of LSALSA. FISTA estimates optimal codes better than LISTA for higher $T$. (\textbf{Right}:) MNIST code prediction error as a function of the inference wallclock time.}
	\label{fig:m1}
\end{figure}
\vspace{-0.3in}
 \small{
 	\begin{table}[htp!]
 		\centering
 			\begin{tabular}{ |c ||c|c|c|c|}
 				\hline
 				\multicolumn{1}{|c||}{ } & \multicolumn{4}{|c|}{Classification Error (in \% )}\\\cline{2-5} 
 				Iter & FISTA & LISTA & SALSA & LSALSA\\
 				\hline
 				1   & 40.87   & 4.58   & 20.73   & \textbf{1.91}		\\
 				\hline
 				5   & 6.10   & 4.65   & 4.85   & \textbf{1.78}		\\
 				\hline
 				7   & 3.43   & 2.04   & 0.81   & \textbf{0.46}		\\
 				\hline
 				10   & 2.03   & 1.38   & \textbf{0.10}   & 0.10		\\
 				\hline
 				15   & 1.07   & 0.88   & 0.02   & \textbf{0.01}		\\
 				\hline
 				50   & 0.02   & 0.61   & \textbf{0.00}   & \textbf{0.00}		\\
 				\hline
 				100   & \textbf{0.0000}   & 0.42   & \textbf{0.00}   & \textbf{0.00}		\\
 				\hline	
 		\end{tabular}
 		\caption{\label{tab: mnist1d} MNIST classification results (10 classes). The best performer is in bold.}
 		\label{tab:m2}
 \end{table}}
\vspace{-0.4in}
\begin{figure}[H]
    \centering
    \includegraphics[width=0.6\textwidth]{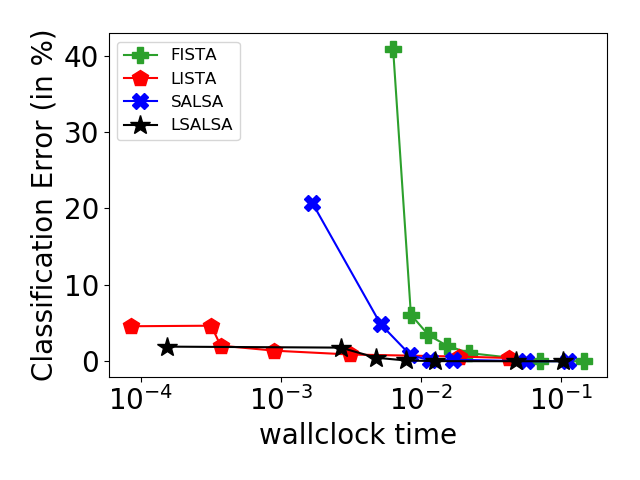}
    \caption{MNIST: Classification errors Vs sparse code inference time plotted for different network lengths.}
    \label{fig: m2} 
\end{figure}

\FloatBarrier
\vspace{-0.5in}
\subsection{CIFAR-10}\label{supp:cifar10}

In CIFAR-10 experiments, $32\times32$ natural images were first converted to grayscale, scaled to values in range $[0, 1]$, and broken  down to $10\times10$ non-overlapping patches. Each image resulted in $9$ patches. Then optimal codes were computed on these patches in similar fashion as described above for MNIST data set, using L1-parameter $\alpha^*=3.0.$
\vspace{-0.2in}
\begin{figure}[H]
	\centering
	\subfigure{\includegraphics[width = 0.49\textwidth]{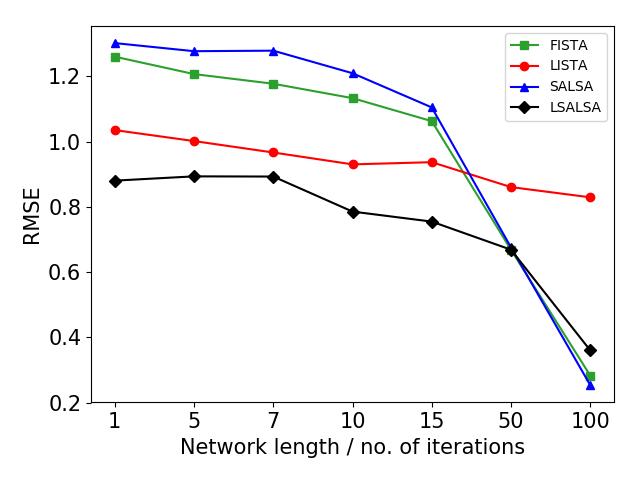}}
	\hfill
	\subfigure{\includegraphics[width = 0.49\textwidth]{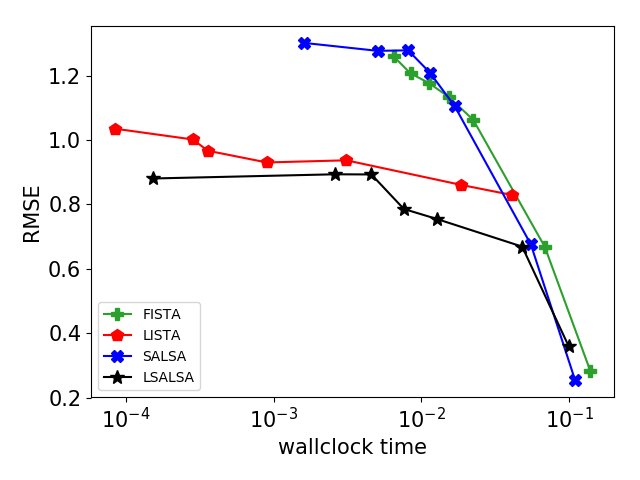}}
	\hfill
	\caption{\label{supp:CIFAR_1D_res}(\textbf{Left:}) CIFAR-10 code prediction errors for varying number of iterations. All methods except LISTA are converging fast after $T = 50$ on this data set. LISTA, FISTA, and SALSA took more than $15$ iterations to produce error obtained by LSALSA in only $1$ iteration. (\textbf{Right:}) CIFAR-10 code prediction error as a function of the inference wallclock time.}
	\label{fig:c1}
\end{figure}
\vspace{-0.3in}
\small{
 \begin{table}[ht]
 	\centering
 	{%
 		\begin{tabular}{ |c ||c|c|c|c|}
 			\hline
 			\multicolumn{1}{|c||}{ } & \multicolumn{4}{|c|}{Classification Error (in \%)}\\\cline{2-5} 
 			Iter & FISTA & LISTA & SALSA  & LSALSA 	\\
 			\hline
 			1   & 86.86   & 79.13   & 89.07   & \textbf{64.69} \\
 			\hline	
 			5   & 82.33   & 76.26   & 87.27   & \textbf{66.31} \\
 			\hline	
 			7   & 79.47   & 74.10   & 82.71   & \textbf{64.64} \\	
 			\hline
 			10   & 75.52   & 71.65   & 82.83   & \textbf{54.98} \\	
 			\hline
 			15   & 70.19   & 72.45   & 75.41   & \textbf{54.99} \\	
 			\hline
 			50   & \textbf{43.14}   & 66.34   & 43.61   & 49.41 \\	
 			\hline
 			100   & 67.86   & 60.22   & \textbf{10.48}   & 18.44 \\	
 			\hline
 	\end{tabular}}
	
 	\caption{\label{tab: cifar1d} CIFAR-10 classification results (10 classes). The best performer is in bold.}
 	\label{tab:c2}
 \end{table}}
\vspace{-0.5in}
\begin{figure}[H]
    \centering
    \includegraphics[width=0.6\textwidth]{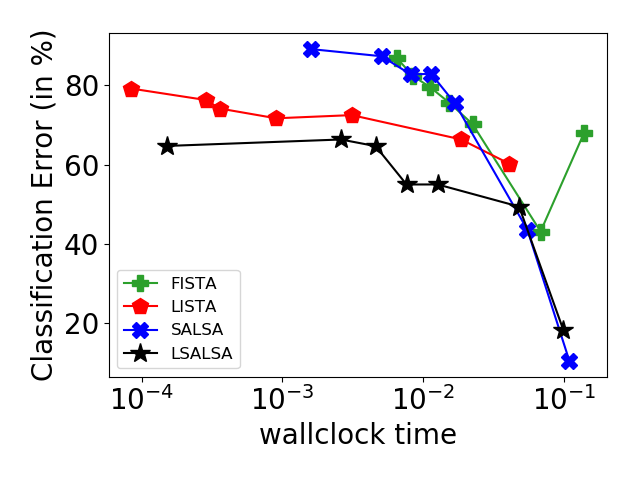}
    \caption{CIFAR-10: Classification errors Vs sparse code inference time plotted for different network lengths.}
    \label{fig: c2} 
\end{figure}

\FloatBarrier
\section{Additional MCA experiments}\label{sup:AddiMCA_res}
\label{sec:D}

\subsection{MNIST + CIFAR}
\vspace{-0.4in}
\begin{figure}[H]
	\centering
	\subfigure[]{\includegraphics[width = 0.49\textwidth]{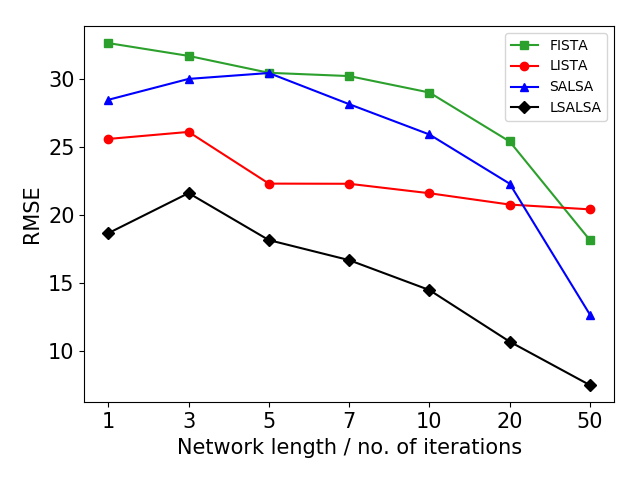}}
	\hfill
	\subfigure[]{\includegraphics[width = 0.49\textwidth]{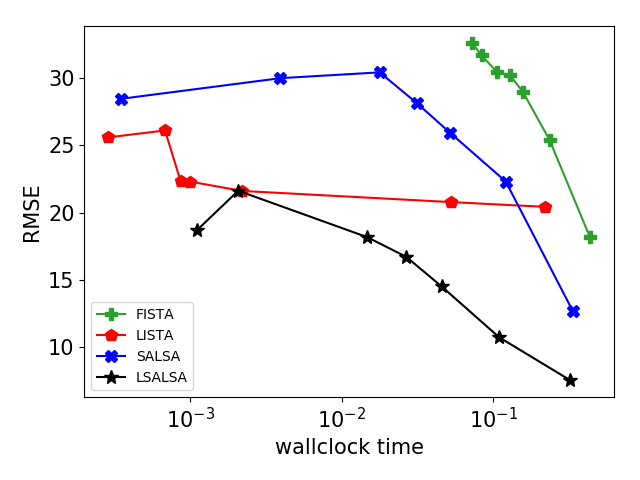}}
	\hfill
	\caption{MCA experiments with MNIST + CIFAR data sets. (a) Code prediction errors for varying numbers of iterations. (b) Code prediction error as a function of the inference wallclock time.}
	\label{fig:mc1}
\end{figure}

MNIST + CIFAR-10 MCA experimental results are summarized here. We combined whole $32\times32$ MNIST digits images with grayscale $32\times32$ CIFAR-10 images and performed MCA-based source separation on them. Optimal codes were computed with parameters $\alpha_1^*=1.5$ (for MNIST), $\alpha_2^*=2.5$ (for CIFAR10), and $\mu^*=10$. Code prediction error curves are presented with respect to the number of iterations $T$ and wallclock time used to make predictions in Figure~\ref{fig:mc1}. The classification results are captured in Table~\ref{tab:mc2} for MNIST codes and Table~\ref{tab:mc3} for CIFAR-10 codes.  

\small{
\begin{table}[H]
	
	\centering 
	\begin{tabular}{ |c ||c|c|c|c|}
		\hline
		\multicolumn{1}{|c||}{ } & \multicolumn{4}{|c|}{Classification Error (in \%)}\\\cline{2-5} 
		Iter & FISTA & LISTA & SALSA  & LSALSA 	\\[1ex]
		\hline
		1   & 66.87   & 37.20   & 33.38   & \textbf{5.88}\\
		\hline	
		3   & 90.00   & 33.22   & 60.73   & \textbf{7.31}\\	
		\hline
		5   & 90.00   & 18.04   & 19.29   & \textbf{4.30}\\	
		\hline
		7   & 90.00   & 15.90   & 8.87   & \textbf{3.21}	\\
		\hline
		10   & 90.00   & 13.59   & 5.36   & \textbf{3.20}\\	
		\hline
		20   & 8.44   & 10.20   & 2.86   & \textbf{4.65}	\\
		\hline
		50   & 21.24   & 6.47   & 12.94   & \textbf{2.98}\\	
		\hline
	\end{tabular}

    \caption{MNIST classification error after source separation (10 classes). The best performer is highlighted in bold.}
	\label{tab:mc2}
\end{table}}
\vspace{-0.2in}
\small{
\begin{table}[H]
	\centering
	\begin{tabular}{ |c ||c|c|c|c|}
		\hline
		\multicolumn{1}{|c||}{ } & \multicolumn{4}{|c|}{Classification Error (in \%)}\\\cline{2-5} 
		Iter & FISTA & LISTA & SALSA  & LSALSA 	\\[1ex]
		\hline
		1   & 88.12   & 87.43   & \textbf{83.95}   & 84.27\\	
		\hline
		3   & 88.73   & 88.15   & 84.42   & \textbf{82.55}\\	
		\hline
		5   & 88.63   & 81.99   & 82.59   & \textbf{74.87}\\	
		\hline
		7   & 88.43   & 82.85   & 81.10   & \textbf{68.21}\\	
		\hline
		10   & 88.85   & 80.08   & 79.02   & \textbf{63.93}\\	
		\hline
		20   & 81.30   & 79.16   & 76.24   & \textbf{57.53}\\	
		\hline
		50   & 70.40   & 81.09   & 74.71   & \textbf{52.60}\\
		\hline	
	\end{tabular}
	
	\caption{CIFAR-10 classification error after source separation (10 classes). The best performer is highlighted in bold. }
	\label{tab:mc3} 
	
\end{table}}

\begin{figure}[ht!]
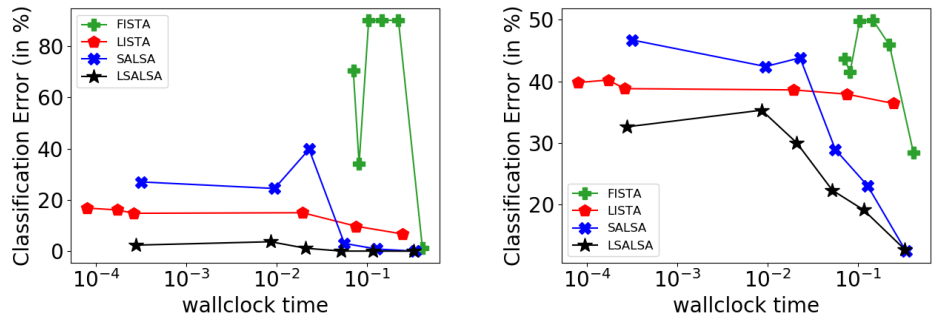

	\centering
	\begin{tabular}{cc}
		\hspace{-0.08in}\includegraphics[width=0.45\textwidth]{MCAimages/2D_time_vs_clerr_MNISTandCD_MNIST.png}&
		\hspace{-0.07in}\includegraphics[width=0.45\textwidth]{MCAimages/2D_time_vs_clerr_MNISTandCD_CD.png}
	\end{tabular}
	\caption{ MCA experiment separating MNIST + CIFAR components: The trade-off between the sparse codes classification error Vs their inference time for different network lengths is captured on (left) for MNIST (right) for CIFAR.}
	\label{fig: mc3 }
\end{figure}

\vspace{-0.5in}
\section{Additional plots: MNIST+ASIRRA}
\label{sec:E}
This Section shows the sparsity/accuracy tradeoff point-cloud plots for MNIST, complementary to Figure~\ref{fig: cdSparsity_multiN} in Section~\ref{subsec:mca}.
\begin{figure}[H]\centering
	\hspace*{\fill}%
	\subfigure[$T=1$]{\includegraphics[width=0.33\textwidth]{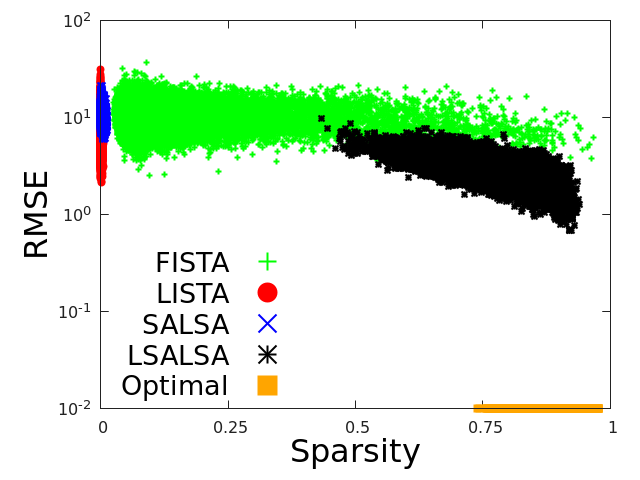}}\hspace{0.1em}%
	\subfigure[$T=3$]{\includegraphics[width=0.33\textwidth]{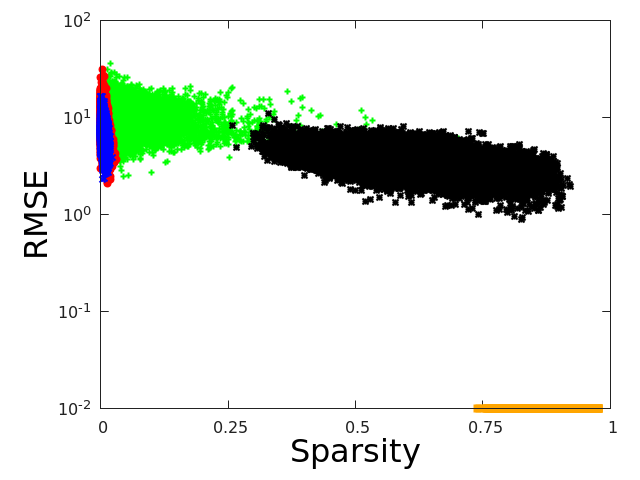}}\hspace{0.1em}%
	\subfigure[$T=5$]{\includegraphics[width=0.33\textwidth]{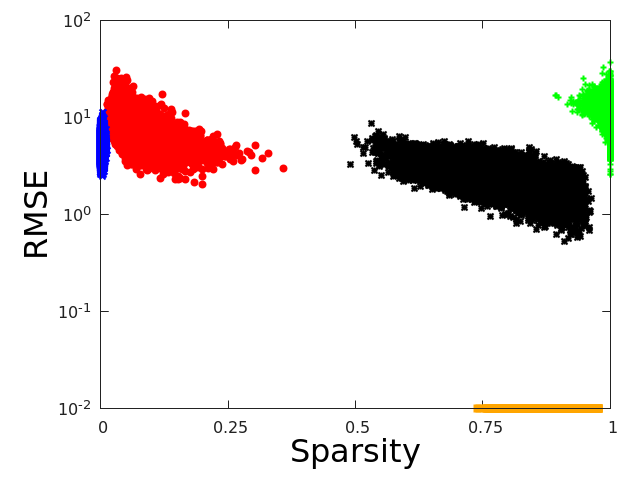}} \hspace*{\fill}%
	\\
	\hspace*{\fill}%
	\subfigure[$T=10$]{\includegraphics[width=0.33\textwidth]{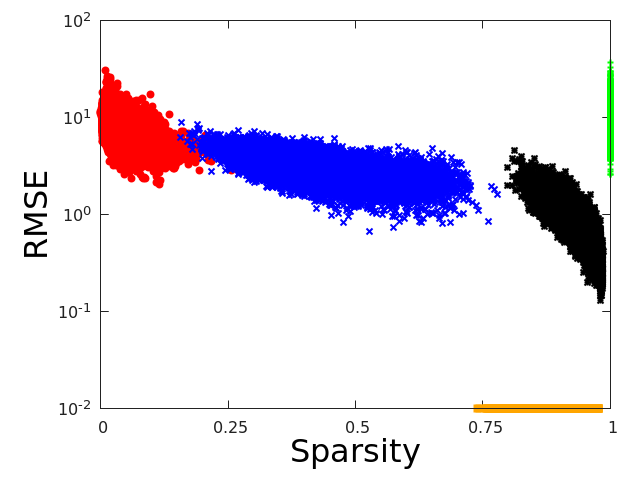}}\hspace{0.1em}%
	\subfigure[$T=20$]{\includegraphics[width=0.33\textwidth]{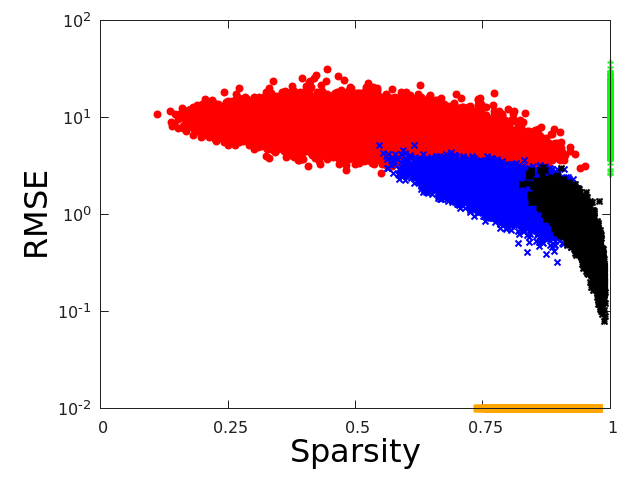}}\hspace{0.1em}%
	\subfigure[$T=50$]{\includegraphics[width=0.33\textwidth]{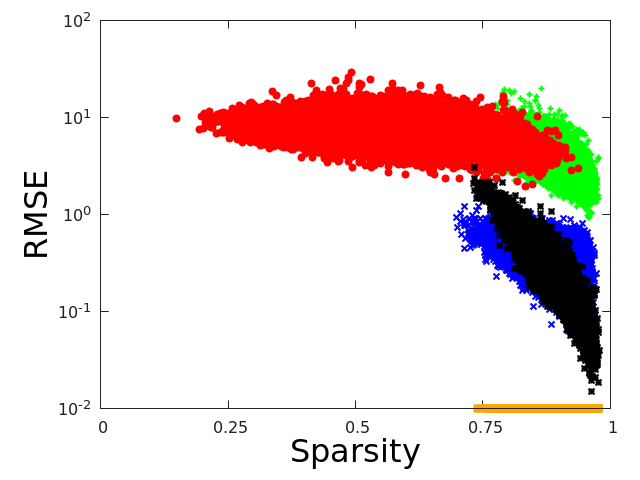}}
	\hspace*{\fill}%
	\caption{Sparsity/accuracy trade-off analysis for MNIST obtained for the source separation experiment with MNIST + ASIRRA data set. Each method corresponds to a colored point cloud, where each point corresponds to one sample from the ASIRRA test data set. LSALSA achieves the best sparsity/accuracy trade-off and is faster than other methods.}
	\label{fig: mnistSparsity_multiN}
\end{figure}

\newpage

\section{Source separation: image reconstruction results}
\label{sec:F}

\begin{figure}[!ht]
	\centering
    \includegraphics[width = \textwidth]{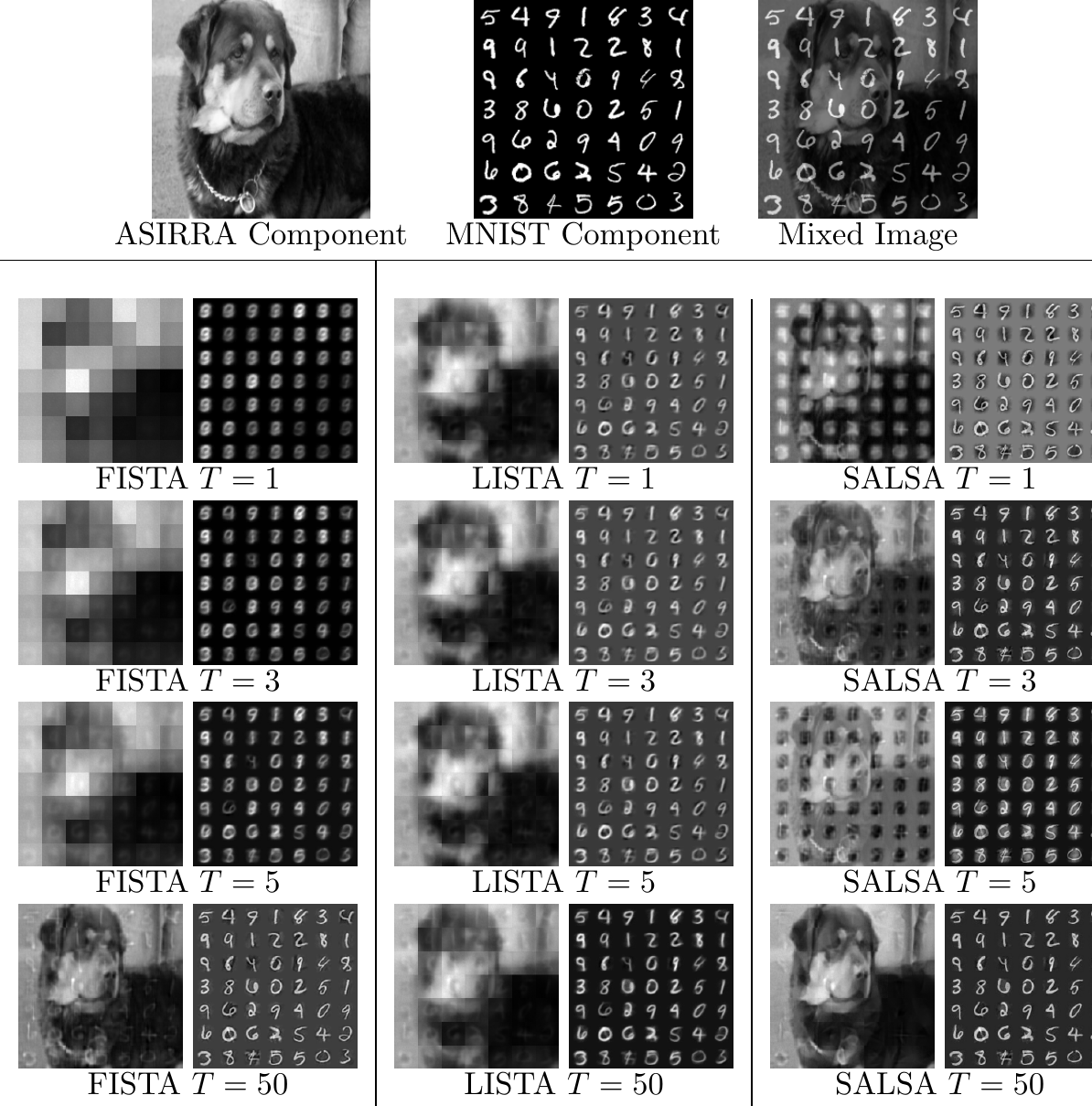}
	\caption{\label{recon: sLs30_SUPP_type2} MCA experiment using MNIST + ASIRRA data set. Image reconstructions obtained by SALSA, LSALSA, FISTA, LISTA for $T = 1,3,5,50$. Top row: original data (components and mixed).}
\end{figure}

\newpage

\begin{figure}[htp!]
	\centering
    \includegraphics[width = \textwidth]{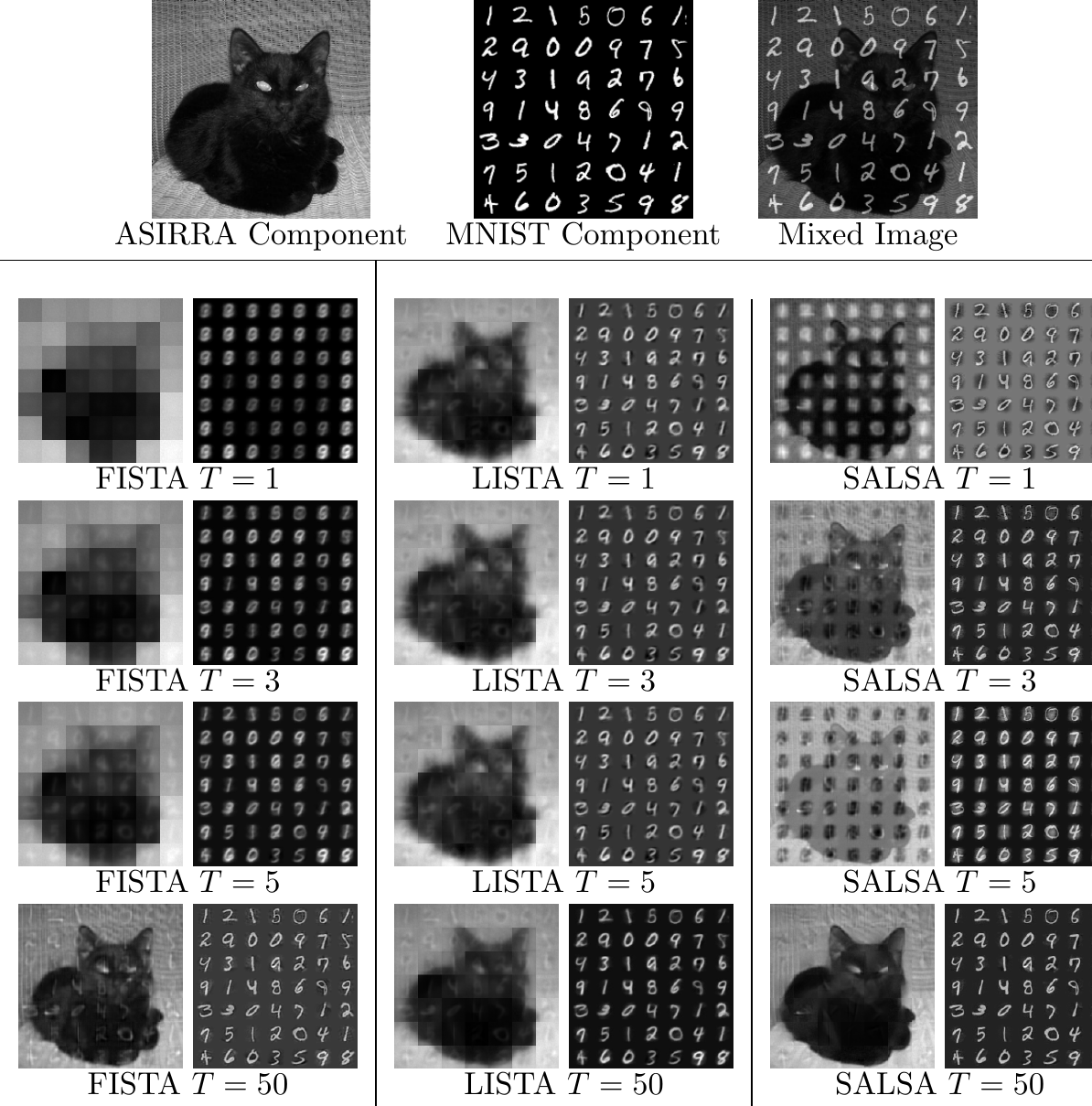}
	\caption{\label{recon: sLs30_SUPP_type2} MCA experiment using MNIST + ASIRRA data set. Image reconstructions obtained by SALSA, LSALSA, FISTA, LISTA for $T = 1,3,5,50$. Top row: original data (components and mixed).}
\end{figure}

\newpage

\begin{figure}[!ht]
	\centering
    \includegraphics[width=0.8\textwidth]{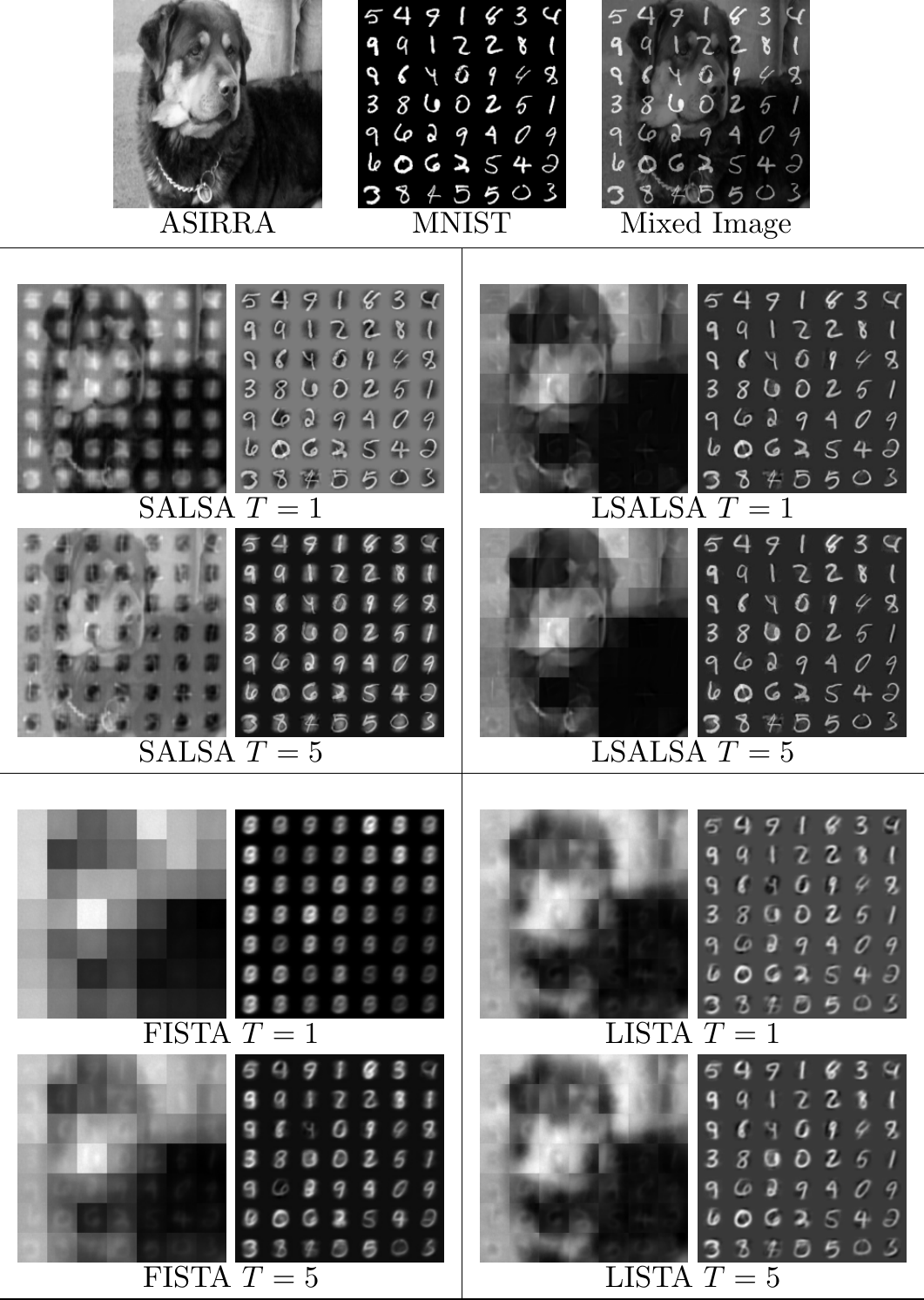}
	\caption{\label{recon: sLs30_SUPP} MCA experiment using MNIST + ASIRRA data set. Image reconstructions obtained by SALSA, LSALSA, FISTA, LISTA for $T = 1,5$. Top row: original data (components and mixed).}
\end{figure}

\newpage

\begin{figure}[!ht]
	\centering
    \includegraphics[width=0.8\textwidth]{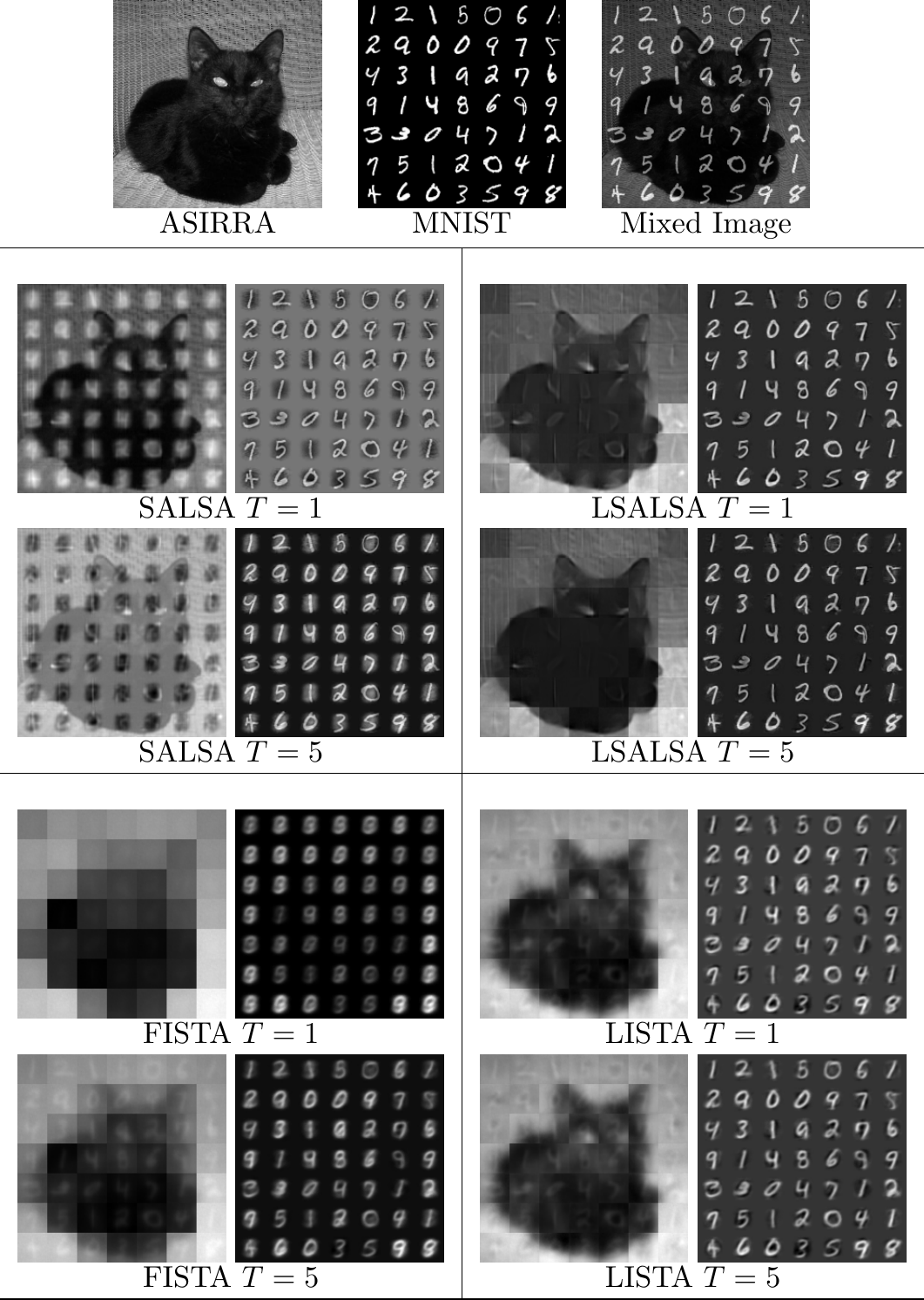}
	\caption{\label{recon: sLs30_SUPP} MCA experiment using MNIST + ASIRRA data set. Image reconstructions obtained by SALSA, LSALSA, FISTA, LISTA for $T = 1,5$. Top row: original data (components and mixed).}
\end{figure}

\begin{figure}[htp!]
	\centering
  \includegraphics[height=0.95\textheight]{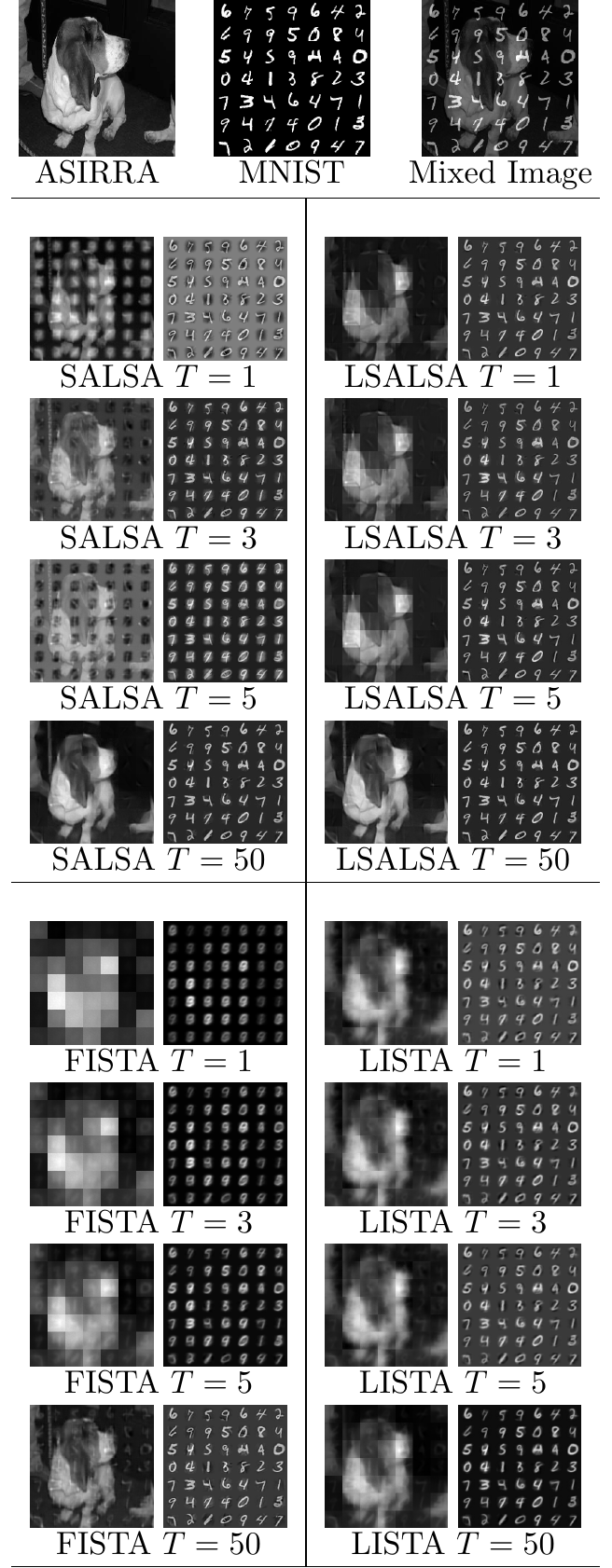}
	\caption{\label{recon: allmini346} MCA experiment. Image reconstructions obtained by SALSA, LSALSA, FISTA, LISTA for $T = 1,3,5,50$. Top row: original data (components and mixed).}
\end{figure}

\end{document}